\documentclass{l4dc2025}

\usepackage{breqn}
\usepackage{cleveref}
\usepackage{wrapfig}
\usepackage{hyperref}
\usepackage{algorithm}
\usepackage{algorithmicx}
\usepackage[noend]{algpseudocode}

\usepackage{graphicx}
\usepackage{Shorthands}
\usepackage{enumitem}
\usepackage{ifthen}
\usepackage{floatflt}
\usepackage[font=small,labelfont=bf]{caption}
\usepackage{array}
\usepackage{amsmath}
\usepackage{booktabs}

\usepackage{hyperref}\hypersetup{colorlinks=true, unicode=true, linkcolor=[rgb]{0.10,0.05,0.67}, citecolor=[rgb]{0.10,0.05,0.67}, filecolor=[rgb]{0.10,0.05,0.67}, urlcolor=[rgb]{0.10,0.05,0.67}}

\usepackage[noindentafter]{titlesec}
\titlespacing\section{0pt}{10pt}{4pt}
\titlespacing\subsection{0pt}{4pt}{3pt}

\usepackage{times}

\usepackage{thmtools} 
\usepackage{thm-restate}

\makeatletter
\renewcommand{\jmlrvolume}[1]{}
\renewcommand{\jmlryear}[1]{}
\renewcommand{\jmlrproceedings}[2]{}
\def\ps@jmlrtps{%
  \def\@oddhead{}% Clear header for odd pages
  \def\@evenhead{}% Clear header for even pages
  \def\@oddfoot{}% Clear footer for odd pages
  \def\@evenfoot{}% Clear footer for even pages
}
\makeatother
\DeclareMathOperator{\Err}{Err}

\DeclareMathOperator{\Regret}{Regret}

\makeatletter
\title[Logarithmic Regret for Nonlinear Control]{Logarithmic Regret for Nonlinear Control}
\author{\Name{James Wang} \Email{jwang541@seas.upenn.edu} \\
\Name{Bruce D. Lee} \Email{brucele@seas.upenn.edu} \\
\Name{Ingvar Ziemann} \Email{ingvarz@seas.upenn.edu} \\
\Name{Nikolai Matni} \Email{nmatni@seas.upenn.edu} \\
\addr All authors are with the Department of Electrical and System Engineering at the University of Pennsylvania.
}

\begin{document}

\maketitle

\begin{abstract}

We address the problem of learning to control an unknown nonlinear dynamical system through sequential interactions. Motivated by high-stakes applications in which mistakes can be catastrophic,  such as robotics and healthcare, we study situations where it is possible for fast sequential learning to occur. Fast sequential learning is characterized by the ability of the learning agent to incur logarithmic regret relative to a fully-informed baseline. We demonstrate that fast sequential learning is achievable in a diverse class of continuous control problems where the system dynamics depend smoothly on unknown parameters, provided the optimal control policy is persistently exciting. Additionally, we derive a regret bound which grows with the square root of the number of interactions for cases where the optimal policy is not persistently exciting. Our results provide the first regret bounds for controlling nonlinear dynamical systems depending nonlinearly on unknown parameters. We validate the trends our theory predicts in simulation on a simple dynamical system.

\end{abstract}

% keywords can be removed
% \keywords{First keyword \and Second keyword \and More}

\section{Introduction}

Controlling an unknown nonlinear system through repeated sequential interaction is a fundamental problem in controls and reinforcement learning. Recent years have seen considerable impact of this paradigm in application areas ranging from  walking robots \citep{yang2020data}, mastering games such as go and StarCraft \citep{silver2017mastering} and even fine-tuning large language models \citep{ouyang2022training}. Problems of this form are often analyzed through the lens of  Markov Decision Processes (MDP). Indeed, there is a wealth of literature on analyzing interactive sequential decision making in tabular MDPs \citep{burnetas1997optimal,dann2015sample}. Extensions to this framework, typically motivated by studying large state and action spaces together with function approximation, are also abundant in the literature \citep{jiang2017contextual, zhou2021nearly}.

However, many problems, including certain robotics and healthcare tasks, are more naturally cast through the framework of continuous control. Such problems can be converted to tabular MDPs through discretization of the state and action spaces; however, doing so often results in intractable reinforcement learning problems. Conversely, the continuous control problem can be solved efficiently in special cases, such as the linear quadratic regulator (LQR) \citep{dean2020sample}. Of the above motivating examples, robotic tasks in particular are plagued by costly data-collection \citep{kober2013reinforcement}. % and the environments may even be previously unseen. 
A similar situation arises in healthcare: giving the wrong treatment doses of a medicine repeatedly can have dire consequences. Consequently in these applications one would hope to find \emph{fast learning algorithms}  that require as few interactions as possible with the unknown system to meet the desired performance criteria. 

In the sequel, we measure the performance of an interactive sequential decision-maker by its regret---its performance as compared to the best policy (in a certain class), in hindsight. A fast learning algorithm in such sequential decision making tasks is characterized as one that attains \emph{regret scaling logarithmically in the number of interactions with the unknown environment}. There has been a wealth of literature in characterizing when such rates are achievable in the setting of bandits \citep{lai1985asymptotically,garivier2019explore} and analogs for tabular reinforcement learning \citep{burnetas1997optimal, ok2018exploration, xu2021fine}. However, to date there has been no general characterization of when this is achievable in continuous control for nonlinear systems with nonlinear dependence on the unknown parameters. We thus ask: \textbf{are there conditions under which such fast learning algorithms exist for continuous control of nonlinear systems with nonlinear parameter dependencies?}

%In this work, we seek 

\subsection{Contribution}
Our main result answers the question of achievability of logarithmic regret in the affirmative.

\begin{theorem}[Informal version of the main result]
    If the optimal policy solving a given continuous control task is identifiable from an experiment running the optimal policy, polylogarithmic regret is attained by our Algorithm \ref{alg:LogRegret}.
\end{theorem}

The crux of our contribution is thus to establish a natural condition for logarithmic regret in nonlinear control problems and to provide a novel algorithm leveraging this condition which achieves logarithmic regret.  To the best of our knowledge, this is the first algorithm achieving (poly-)logarithmic regret in general nonlinear control problems. 

The intuition behind our result is as follows. If the data collected by running the optimal policy is sufficiently informative about the unknown parameters, then it is unnecessary to inject exploratory noise to perform online control. In particular, a policy which is near optimal will enjoy similarly informative data collection, allowing the learner to gradually approach the optimal policy by playing certainty equivalent controllers synthesized with estimates of the dynamics parameters. We formalize this intuition with a persistence of excitation condition, asking that the Fisher information matrix of the optimal policy is positive definite. 

Finally, for completeness, we also provide an algorithm attaining sublinear regret in the absence of our identifiability condition. This result can be found in \Cref{s: slow learning} along with all proofs.

\subsection{Related Work}

\paragraph{Logarithmic Regret in Bandits and RL}

The question of whether logarithmic regret is attainable or not is intimately connected with the exploration exploitation trade-off. Beginning with \cite{lai1985asymptotically} in the tabular bandit setting, \emph{gap-dependent} regret bounds have been established showing that logarithmic regret is possible whenever there is a strict separation between the reward of the optimal action and that of a second best, or worse, action. Similar gap sufficient conditions for logarithmic regret also exist in tabular reinforcement learning~\citep{burnetas1997optimal,ok2018exploration,xu2021fine}. In the worst case, or for instance in linear bandits where there is no gap, logarithmic regret is impossible and instead regret scales with the square root of the number of interactions with unknown environment~\citep[cf. e.g.,][]{garivier2019explore}.

\paragraph{Closed-Loop Identifiability and Adaptive Control}

Within the system identification community, the exploration-exploitation trade-off is often referred to as the \emph{dual nature of control} \citep{feldbaum1960dual1,feldbaum1960dual2} and is related to issues of \emph{closed-loop identifiability}  \citep{ljung1999system}. Roughly speaking, closed-loop identifiability issues arise because a fixed control law might not sufficiently excite the system under consideration in the necessary directions in state space (or feature space more generally). Indeed, in the Linear Quadratic Regulator (LQR) setting, \citet{polderman1986necessity} gives an elegant geometric argument showing that the true parameters need to be identified in order to ascertain the optimal control law. It is also interesting to note that, precisely because the minimum variance controller is closed-loop identifiable \citep{lin1985will} (in contrast to the more general LQR controller), logarithmic regret can be achieved in this setting~\citep{lai1986asymptotically}. Reiterating the point above: the reason for the impossibility of pure exploitation is precisely a lack of closed-loop identifiability. This insight is leveraged in \cite{simchowitz2020naive} and \cite{ziemann2024regret} to show logarithmic regret is impossible in general in the linear quadratic Gaussian control problem. However, given some prior information about the system (e.g. if the way the input impacts the state transitions is known), then closed-loop identifiability may hold, making logarithmic regret achievable for LQR \citep{cassel2020logarithmic, jedra2022minimal, lee2024nonasymptotic}. Alternatively, if the policy choice is restricted to a set in which all possible candidate provide closed-loop identifiability of the system parameters, then \citet{lale2020logarithmic} demonstrate logarithmic regret for the Linear Quadratic Gaussian (LQG). 

Closed-loop identifiability issues similarly hinder the achievability of logarithmic regret in the online control of nonlinear systems. In the setting of nonlinear dynamical systems which depend linearly on some unknown parameters, \citet{kakade2020information, boffi2021regret} propose algorithms that achieve regret scaling with the square root of the number of interactions. \cite{lale2024falcon} consider linear function approximators for smooth systems, and provide an algorithm achieving regret scaling with the square root of the number of interactions in general, and logarithmic regret if the system is sufficiently smooth. Critically, as with \citet{lale2020logarithmic}, \citet{lale2024falcon} assume that all policies in the policy class provide closed-loop identifiability of the parameters.  By contrast, we do not assume a priori access to a policy yielding such identifiability; we show that it suffices that the \emph{unknown} optimal policy yields easy identification and our algorithm then adapts to this property. Moreover, we consider dynamical systems which depend nonlinearly on an unknown parameter, and propose an algorithm that incurs logarithmic regret as long as the optimal policy enables closed-loop identification.

\paragraph{Learning in Dynamical Systems}

Our contribution also draws on a recent line of work on learning in dynamical systems beginning with \cite{simchowitz2018learning, faradonbeh2018finite}. The authors therein show that non-asymptotic parameter recovery from a single trajectory is possible in certain marginally stable, or unstable, linear dynamical systems. \cite{mania2022active} leverage the parameter recovery bounds to enable efficient exploration. Non-asymptotic identification of more general nonlinear systems is studied by \citet{sattar2022non, foster2020learning, ziemann2022learning}. \citet{treven2023optimistic,wagenmaker2024optimal, lee2024active} study control-oriented experiment design in an episodic setting for nonlinear systems.

\subsection{Notation}

The Jacobian of a vector-valued function $g: \mathbb R^n \to \mathbb R^m$ is denoted $D g$, and follows the convention for any $x\in\mathbb R^n$, the rows of $D g(x)$ are the transposed gradients of $g_i(x)$. The $p^{th}$ order derivative of $g$  is  denoted by $D^{(p)} g$. Note that for $p \geq 2$, $D^{(p)} g(x)$ is a tensor for any $x\in\mathbb R^{n}$. The operator norm of such a tensor is denoted by $\norm{D^{(p)} g(x)}_{\mathsf{op}}$. 
 For a function $f: \mathsf{X} \to \mathbb R^{d_y}$, we define $\norm{f}_{\infty} \triangleq \sup_{x \in \mathsf{X}} \norm{f(x)}$. A Euclidean norm ball of radius $r$ centered at $x$ is denoted $\calB(x,r)$. 

\section{Problem Formulation}

We consider a nonlinear dynamical system given by the dynamics\begin{align}\label{expr:Dynamics}
    x_{t + 1} = f(x_t, u_t, \phi^*) + w_t, \quad t = 1, ..., T - 1
\end{align}
where the state $x_t \in \mathbb{R}^{d_x}$; the input $u_t \in \mathbb{R}^{d_u}$; and the additive noise $w_t \in \mathbb{R}^{d_x}$, with $w_t \overset{i.i.d.}{\sim} \mathcal N(0, \sigma^2I)$. Let $x_1 \in\mathbb{R}^{d_x}$ be arbitrary. Here, $f$ is the dynamics function and depends on a parameter $\phi^* \in \mathbb{R}^{d_\phi}$. We assume that there exists some positive $B$ such that $\norm{\phi^*} \leq B$ and $\norm{f(\cdot, \cdot, \phi)}_\infty \leq B$ for all $\phi \in \mathbb{R}^{d_\phi}$ satisfying $\norm{ \phi } \leq B$.

We study an online learning problem under these dynamics. We consider a learner who has knowledge of the dynamics $f$, but not the parameter $\phi^*$. In each episode $n = 1, \dots, N$, the learner executes a policy $\pi_n$ from the set of policies $\curly{\pi_0} \cup \Pi$, where $\pi_0$ is an initial (possibly randomized) exploration policy, while $\Pi$ is a class of deterministic controllers which take as input a point $x\in \mathbb{R}^{d_x}$ and return a control input $u\in\mathbb{R}^{d_u}$. Then, the learner observes a trajectory $(x_1, u_1), \dots (x_T, u_T)$ (generated by unrolling \eqref{expr:Dynamics} with $u_t \sim \pi_n(x_t)$); and incurs the cost $J(\pi_n, \phi^*)$, where \begin{align}\label{expr:Task}
    J(\pi, \phi) := \mathbb{E}_{\pi}^{\phi}\brac{\sum_{t = 1}^{T}{c_t(x_t, u_t)}}
\end{align}
for some cost functions $\{c_t\}_{t = 1, ..., T}$ which are fixed across episodes. The subscript on the expectation denotes that the policy $\pi$ is played, while the superscript denotes that the dynamics (\ref{expr:Dynamics}) are rolled out under $\phi$. The expectation is taken over the noise $w_t$ and the policy $\pi_n$. We suppose that the policy class $\Pi$ is parametric: $\Pi = \{\pi_\theta : \theta \in \mathbb{R}^{d_\theta}\}$. 

The learner's objective is to achieve a low sum of costs over episodes. A natural metric is therefore to minimize the regret, defined as \begin{align}
    \Regret(N) &:= \left(\sum_{n=1}^{N}{J(\pi_n, \phi^*)}\right) - N\min_{\pi \in \Pi}{J(\pi, \phi^*)}.
\end{align}
We will explore no-regret learners for this setting, for which $\Regret(N) / N \to 0$ as $N \to \infty$. 

\subsection{Certainty Equivalent Control}

Since the learner does not have access to the true dynamics $\phi^*$, it cannot directly solve a policy optimization problem under the system $f(x, u, \phi^*)$ \citep{fazel2018global} for the optimal controller. Instead, we leverage the principle of certainty equivalence. In particular, the learner uses the data collected from its interactions to pose an estimate $\hat\phi$ for the parameter $\phi^\star$. Using this estimate, the learner solves the policy optimization problem,\begin{align}\label{expr:PolicyOptimization}
    \theta^*(\hat \phi) \in \argmin_{\theta \in \mathbb{R}^{d_\theta}} J(\pi_\theta, \hat\phi).
\end{align}
The certainty equivalent policy may then be expressed as a function of the estimate $\hat \phi$ as \begin{align}
    \pi^*(\hat \phi) \triangleq \pi_{\theta^*(\hat\phi)}.
\end{align}
Under some additional assumptions (which we will lay out in the following section), we can measure the performance of the policy $\pi^*(\hat \phi)$ in terms of the quality of the estimate $\hat \phi$. Here, we define the prediction error of $\hat \phi$, on trajectories collected using a policy $\pi$ (which is not necessarily $\pi^*(\hat \phi)$), as\begin{align}
    \Err_{\pi}^{\phi^*}(\phi) \triangleq \mathbb{E}_{\pi}^{\phi^*}\brac{ \frac{1}{T} \sum_{t=1}^{T}{ \norm{ f(x_t, u_t, \phi) - x_{t+1} }^2 } }.
\end{align}

\subsection{Assumptions}

In order to relate the excess cost achieved by a certainty equivalent controller synthesized under a dynamics estimate $\phi$ to the quality of the estimate $\hat \phi$, we impose some smoothness assumptions on the dynamics and policy class.

\begin{assumption} \label{assume:SmoothDynamics}
    (Smooth dynamics). The dynamics are four times differentiable with respect to $u$ and $\phi$. Furthermore, for all $(x, u) \in \mathbb{R}^{d_x} \times \mathbb{R}^{d_\phi}$, and $i, j\in \{0, 1, 2, 3\}$ such that $1 \leq i + j \leq 4$, the derivatives of $f$ satisfy $\norm{D_{\phi}^{(i)}D_u^{(j)} f(x, u, \phi)}_{\mathsf{op}} \leq L_f$.
\end{assumption}

\begin{assumption} \label{assume:SmoothPolicyClass}
    (Smooth exploitation policy class). For all policies $\pi \in \Pi$ and $x\in \mathcal X$, the function $\pi_\theta(x)$ is four-times differentiable in $\theta$. Furthermore $\norm{D_\theta^{(i)}\pi_\theta (x)}_{\mathsf{op}} \leq L_{\Pi}$ for all $i = 1, ..., 4$, all $\theta \in \mathbb R^{d_\theta}$, and all $x\in \mathcal X$.
\end{assumption}

We additionally require that the costs are bounded for policies in the class $\curly{\pi_0} \cup \Pi$ and all dynamics parameters in a neighborhood of the true parameter. Intuitively, this allows our learning algorithm to occasionally play bad policies without incurring too much excess cost.
 
\begin{assumption} \label{assume:BoundedCosts}
    (Bounded costs). There exists $r_{\mathsf{cost}} > 0$ such that for all $\phi \in\calB(\phi^*, r_{\mathsf{cost}})$, and all $\pi \in \curly{\pi_0} \cup \Pi$, we have $\mathbb E_{\pi}^\phi \brac{\paren{\sum_{t=1}^T{c_t(x_t, u_t)} }^2} \leq T^2 L_{\mathsf{cost}}^2$. (Together with Jensen's inequality, this immediately implies that $\mathbb E_{\pi}^\phi \brac{ \sum_{t=1}^T{c_t(x_t, u_t)} } \leq T L_{\mathsf{cost}}$ for such $\phi$ and $\pi$.)
\end{assumption}
As the task is episodic, the above assumption holds if the stage costs are uniformly bounded for all $x \in \mathbb R^{d_x}$ and $u \in \mathbb R^{d_u}$. Alternatively, if the stage costs are smooth, the above condition holds if the states and inputs are bounded with high probability. This is satisfied for $\Pi$ by the smoothness of the dynamics (Assumption \ref{assume:SmoothDynamics}) and exploitation policy class (Assumption \ref{assume:SmoothPolicyClass}). A mild assumption that the initial policy $\pi_0$ plays bounded inputs suffices to guarantee the above condition also holds for $\pi_0$.

We additionally suppose that the certainty equivalent controller parameters, as a function of the estimated dynamics $\phi$, are locally smooth near the true dynamics $\phi^*$.

\begin{assumption} \label{assume:WagenmakerProp6} There exists some $r_{\mathsf{ce}} > 0$ such that for all $\phi \in\calB(\phi^*, r_{\mathsf{ce}})$,\begin{itemize}
    \item $\nabla_{\theta} J(\pi_\theta, \phi) \mid_{\theta = \theta^*(\phi)} = 0$,

    \item $\theta^*(\phi)$ is three times differentiable and $\left\| D_{\phi}^{(i)} \theta^*(\phi) \right\|_{\mathsf{op}} \leq L_{\mathsf{ce}}$ for some $L_{\mathsf{ce}} > 0$ and $i\in \{1, 2, 3\}$.
\end{itemize}
\end{assumption}
It is shown in Proposition 6 of \citet{wagenmaker2024optimal} that this condition holds if the minimizer of $J(\pi_{\theta}, \phi_\star)$ is unique, and $\nabla_{\theta}^2 J(\pi_{\theta}, \phi_\star) \succ 0$. 

In order to bound the parameter recovery error in terms of the prediction error, additional identifiability conditions are needed. \cite{ziemann2024sharp} show that a rather minimal Lojasiewicz condition \citep[cf.][]{roulet2017sharpness} relating the sharpness of an objective to its manifold of minimizers is sufficient for learning from dependent data. The following definition of a Lojasiewicz policy is taken from \cite{lee2024active} and extends the corresponding definition from \cite{ziemann2024sharp} to decision-making. In the setting of \cite{lee2024active}, the following definition of a Lojasiewicz policy bounds the estimation error $\norm{\phi - \phi^*}$ as a function of the prediction error $\Err_{\pi}^{\phi^*} (\phi)$ for all dynamics parameters $\phi$.

\begin{definition}\label{def:Lojasiewicz}
    For positive numbers $C$ and $\alpha$, say that a policy $\pi \in \Pi$ is $(C, \alpha)$-Lojasiewicz if \begin{align*}
        \norm{\hat\phi - \phi^* } \leq C \Err_{\pi}^{\phi^*} (\hat \phi)^\alpha \quad\text{for all $\hat \phi \in \mathbb R^{d_\phi}$.}
    \end{align*}
\end{definition}

Next, to ensure parameter recovery is possible for the learner, we make the following assumption regarding identifiability.

\begin{assumption} \label{assume:InitialLojasiewiczPolicy}
    (Initial Lojasiewicz policy). Fix some positive constant $C_{\mathsf{Loja}}$ and $\alpha \in (1/4, 1/2]$. The learner has access to a policy $\pi_0$ which is $(C_{\mathsf{Loja}}, \alpha)$-Lojasiewicz (here, we do not require that $\pi_0 \in \Pi$; furthermore, we allow $\pi_0$ to be randomized).
\end{assumption}
This is satisfied in linear systems with $\alpha = 1/2$ if the initial controller $\pi_0$ plays Gaussian noise as input, and both the controller noise and process noise have positive definite covariance matrices. More generally, Theorem 2 of \cite{musavi2024identification} implies that playing bounded i.i.d. random inputs suffices to satisfy this condition with $\alpha = 1/2$ in a broad class of analytic nonlinear systems (although the constant $C_{\mathsf{Loja}}$ may be very high).

While Assumption \ref{assume:InitialLojasiewiczPolicy} ensures that the learner can identify the true dynamics $\phi^*$ using only data collected under $\pi_0$, the rate of recovery may be slow under only the assumptions listed previously. In order to obtain polylogarithmic  regret bounds, we require the assumption that the optimal controller, defined by $\theta^* \triangleq \argmin_{\theta} J(\pi_{\theta}, \phi_\star)$, is persistently exciting. Persistence of excitation for a nonlinear dynamical system involves the positive definiteness of the matrix\begin{align*}
    \Sigma^\pi \triangleq \mathbb{E}^{\phi^*}_{\pi} \brac{\frac{1}{T} \sum_{t = 1}^{T}{Df(x_t, u_t, \phi^*)^\top Df(x_t, u_t, \phi^*)} } = D^{(2)}_{\phi}\paren{\Err_{\pi}^{\phi^*} (\phi)} \bigg|_{\phi = \phi^*} 
\end{align*} 
where $Df(x_t, u_t, \phi^*)$ denotes the Jacobian of $f$ with respect to $\phi$ evaluated at $\phi^*$. It can be shown that $\Sigma^\pi$ is a positive scalar multiple of the Fisher Information matrix (when the system evolves according to $\phi^*$ and $\pi$) and hence this condition is equivalent to requiring the positive definiteness of this Fisher Information matrix.

\begin{assumption} \label{assume:PersistencyOfExcitationOptimalController}
    (Persistency of excitation for the optimal controller). The optimal policy under the true dynamics $\phi^*$, denoted $\pi_{\theta^*} \triangleq \pi_{\theta^*(\phi^*)}$, is persistently exciting, i.e. for some $\mu> 0$,\begin{align*}
        \mathbb{E}_{\pi_{\theta^*}}^{\phi^*}\left[ \frac{1}{T}\sum_{t = 1}^{T}{Df(x_t, u_t, \phi^*)^\top Df(x_t, u_t, \phi^*)} \right] \succeq \mu I_{d_\phi}.
    \end{align*}
\end{assumption}
Note that the above assumption is not satisfied in LQR in general when both the $A^*$ and $B^*$ matrices are unknown. However, \cite{lee2024nonasymptotic} show that a sufficient condition for Assumption \ref{assume:PersistencyOfExcitationOptimalController} to hold in linear systems is that either 1) the $A^*$ matrix is known and the optimal controller $K^*$ has full row rank or 2) the $B^*$ matrix is known. 

Finally, we reiterate that in the event that \Cref{assume:PersistencyOfExcitationOptimalController} does not hold, we can obtain slower, but still sublinear regret rates under very general conditions. See Appendix A of the extended manuscript for details.

\section{Fast Learning}\label{s: fast learning}

Under Assumptions \ref{assume:SmoothDynamics}, \ref{assume:SmoothPolicyClass}, \ref{assume:BoundedCosts}, \ref{assume:WagenmakerProp6}, \ref{assume:InitialLojasiewiczPolicy}, and \ref{assume:PersistencyOfExcitationOptimalController}, we give an algorithm (\Cref{alg:LogRegret}) based on the aforementioned certainty equivalence principle which achieves polylogarithmic regret in our online nonlinear control setting.
% Algorithm \ref{alg:LogRegret} is motivated by a reduction to online strongly convex optimization.
Given an initial Lojasiewicz policy $\pi_0$, the exploitation policy class $\Pi$, the number of episodes $N$, the number of initial phase episodes $N_{\mathsf{phase\,1}}$ (where $0 \leq N_{\mathsf{phase\,1}} \leq N$), and a confidence radius $r_{\Phi}$, the algorithm proceeds in two phases. 

In the first phase, the learner collects a dataset $\{ (x_{n, t}, u_{n, t}, x_{n, t+1}) \}_{t = 1, ..., T}^{n = 1,..., N_{\mathsf{phase\,1}}}$ using $\pi_0$, and finds an confidence ball $\Phi$ with radius $r_{\Phi}$ such that $\phi^* \in \Phi$ with high probability. The confidence ball is centered at $\phi_0$, which is the solution to a nonlinear least squares problem, \begin{align}\label{expr:NonlinearLeastSquares}
    \phi_0 \in \argmin_{\phi \in \mathbb{R}^{d_\phi}, \norm{\phi}\leq B} \sum_{n = 1}^{N_{\mathsf{phase\,1}}}{ \sum_{t=1}^{T}{ \norm{ x_{n, t + 1} - f(x_{n, t}, u_{n, t}, \phi)}^2 } }.
\end{align}
With a sufficiently small $r_{\Phi}$, and conditioned on the event $\phi^* \in \Phi$, we show that policies synthesized using estimates that fall within this set enjoy a positive definite Fisher Information; equivalently, the prediction error $\Err_{\pi}^{\phi^*}(\phi)$ is strongly convex on $\Phi$ for all controllers $\pi$ synthesized with dynamics estimates $\hat \phi \in \Phi$. This motivates an online convex optimization procedure in the second phase.

In the second phase, the learner interacts with the system by playing policies synthesized using parameter estimates from $\Phi$. More specifically, the learner uses the certainty equivalent policy $\pi$ corresponding to its current estimate of $\phi^*$ to collect a single trajectory $\mathcal D = \curly{(x_t, u_t, x_{t + 1})}_{t = 1, \dots, T}$. The mean-squared-error of a dynamics estimate $\phi$ on the dataset $\mathcal D$ is \begin{align}\label{expr:EmpiricalPredictionError}
    l_{\mathcal D}(\phi) \triangleq \frac{1}{T}\sum_{t=1}^{T}{ \left\| f(x_t, u_t, \phi) - x_{t+1} \right\|^2 }
\end{align}
and the learner updates its estimate of $\phi^*$ using the gradient $ \nabla l_{\mathcal D}(\phi)$, and repeats this process.

\begin{algorithm}
\caption{Continuous Refinement}\label{alg:LogRegret}
\begin{algorithmic}[1]
\Require Exploration policy $\pi_0$, exploitation policy class $\Pi$, number of episodes $N$, number of initial phase episodes $N_{\mathsf{phase\,1}}$, confidence radius $r_{\Phi}$, hyperparameter $\mu$
\State Play $\pi_0$ for $N_{\mathsf{phase\,1}}$ episodes to collect the dataset $\mathcal{D}_0 := \{ (x^{(0)}_{n, t}, u^{(0)}_{n, t}, x^{(0)}_{n, t + 1}) \}_{t = 1, ..., T}^{n = 1, ..., N_{\mathsf{phase\,1}}}$ \Comment{First phase}
\State Set $\phi_0$ via least squares \eqref{expr:NonlinearLeastSquares} using $\mathcal{D}_0$
\State $\Phi \gets\calB(\phi_0, r_{\Phi})$
\For{$i = 0, 1, \dots, N - N_{\mathsf{phase\,1}} - 1$}
\Comment{Second phase}
    \State Set $\pi_{i + 1} \gets \pi^*(\phi_{i})$ 
    \State Play $\pi_{i + 1}$ to collect a single trajectory $\mathcal{D}_{i + 1} := \{ (x^{(i + 1)}_{t}, u^{(i + 1)}_{t}, x^{(i + 1)}_{t + 1}) \}_{t = 1, ..., T}$
    \State $\psi_{i + 1} \gets \phi_{i} - \frac{8}{\mu \cdot (i + 1)}\nabla l_{\mathcal{D}_{i + 1}}(\phi_{i})$, with $l_{\mathcal D}$ in \eqref{expr:EmpiricalPredictionError}
    \State $\phi_{i + 1} \gets \argmin_{\phi \in \Phi}\| \phi - \psi_{i + 1} \|$
\EndFor
\end{algorithmic}
\end{algorithm}

In general, the nonlinear least squares problem (\ref{expr:NonlinearLeastSquares}) and policy optimization problem (\ref{expr:PolicyOptimization}) may be computationally challenging. The focus of this work is to understand the statistical complexity of the problem rather than the computational complexity. However, it is worth noting that the online optimization procedure is computationally efficient and therefore the learner may efficiently execute the second phase of the dynamics estimation procedure online. Additionally, for particular systems \eqref{expr:Dynamics} and objectives \eqref{expr:Task}, the policy optimization problem \eqref{expr:PolicyOptimization} may be efficient. This is the case, for instance, if the optimal solution to the policy optimization problem can be achieved via feedback linearization \citep{charlet1989dynamic} by choosing the input to cancel out some portion of the dynamics. We consider such an example in \Cref{s:numerical}. 

% \subsection{Regret Bound}

Our main result bounds the regret incurred by Algorithm \ref{alg:LogRegret} in terms of $N$ and $N_{\mathsf{phase\,1}}$ under the aforementioned smoothness and identifiability conditions.

\begin{theorem}\label{thm:LogRegret}
    Consider applying Algorithm \ref{alg:LogRegret} to the system \eqref{expr:Dynamics} with initial policy $\pi_0$ satisfying Assumption \ref{assume:InitialLojasiewiczPolicy}, policy class $\Pi$ satisfying Assumption \ref{assume:SmoothPolicyClass}, number of iterations $N$, number of initial phase episodes $N_{\mathsf{phase\,1}}$ and confidence radius $r_{\Phi}$. Additionally suppose that the dynamics satisfy Assumption \ref{assume:SmoothDynamics} and that the costs satisfy Assumption \ref{assume:BoundedCosts}. Furthermore, suppose that the dynamics, objective, and policy class satisfy Assumption \ref{assume:WagenmakerProp6}. Finally, suppose that the true optimal controller $\pi_{\theta^*}$ and the hyperparameter $\mu$ satisfy Assumption \ref{assume:PersistencyOfExcitationOptimalController}. Then,\begin{align*}
        \mathbb{E}[\Regret(N)] &\leq \mathsf{poly}_\alpha (d_x, \sigma, L_f, L_{\mathsf{cost}}, \mu^{-1}) T \log N + TN_{\mathsf{phase\,1}}L_{\mathsf{cost}}
    \end{align*}
    as long as the following both hold:\begin{itemize}
        \item $r_{\Phi} \leq \mathsf{poly}\paren{r_{\mathsf{ce}}, r_{\mathsf{cost}}, d_x^{-1}, \sigma^{-1}, L_f^{-1}, L_{\Pi}^{-1}, L_{\mathsf{ce}}^{-1}, \mu} T^{-1/2}$,

        \item $N_{\mathsf{phase\,1}} \geq \mathsf{poly}_\alpha \paren{\log N, \log T, d_x, d_\phi, \sigma, C_{\mathsf{Loja}}, r_{\mathsf{ce}}^{-1}, r_{\mathsf{cost}}^{-1}, L_f, L_\Pi, L_{\mathsf{ce}}, \mu^{-1}, \log B } T^{1 / (2\alpha) - 1}$. Here, the subscript $\alpha$ indicates that the degree of the polynomial depends on $\alpha$.
    \end{itemize}
\end{theorem}
Theorem \ref{thm:LogRegret} states that if $r_{\Phi}$ is chosen small enough and the number of initial phase episodes $N_{\mathsf{phase\,1}}$ exceeds some burn-in which is polylogarithmic in $N$ and polynomial in all other relevant system parameters, then the regret of Algorithm \ref{alg:LogRegret} grows at most linearly with $\log N$ and $N_{\mathsf{phase\,1}}$. Plugging in specific choices for $r_{\Phi}$ and $N_{\mathsf{phase\,1}}$ yields a polylogarithmic regret bound for Algorithm \ref{alg:LogRegret}.

\begin{corollary} \label{cor: hyperparameter choice}
    Suppose we apply Algorithm \ref{alg:LogRegret} in the setting of Theorem \ref{thm:LogRegret} with the parameters:\begin{itemize}
        \item $r_{\Phi} = \mathsf{poly}\paren{r_{\mathsf{ce}}, r_{\mathsf{cost}}, d_x^{-1}, \sigma^{-1}, L_f^{-1}, L_{\Pi}^{-1}, L_{\mathsf{ce}}^{-1}, \mu} T^{-1/2}$,

        \item $N_{\mathsf{phase\,1}} = \mathsf{poly}_\alpha \paren{\log N, \log T, d_x, d_\phi, \sigma, C_{\mathsf{Loja}}, r_{\mathsf{ce}}^{-1}, r_{\mathsf{cost}}^{-1}, L_f, L_\Pi, L_{\mathsf{ce}}, \mu^{-1}, \log B } T^{1 / (2\alpha) - 1}$.
    \end{itemize}
    Then, Algorithm \ref{alg:LogRegret} achives regret depending polylogarithmically on the number of episodes $N$, i.e.,\begin{align*}
         \mathbb E [\Regret(N)] \leq \mathsf{poly}_\alpha \paren{\log N, \log T, d_x, d_\phi, \sigma, C_{\mathsf{Loja}}, r_{\mathsf{ce}}^{-1}, r_{\mathsf{cost}}^{-1}, L_f, L_\Pi, L_{\mathsf{ce}}, L_{\mathsf{cost}}, \mu^{-1}, \log B } T^{1 / (2 \alpha)}.
    \end{align*}
\end{corollary}
The full proof of Theorem \ref{thm:LogRegret} may be found in \Cref{appendix: B}; we provide a brief sketch below.

\begin{proof}[Proof Sketch]
For $N_{\mathsf{phase\textnormal{ }1}}$ satisfying the given bound, the system identification results of \citet{ziemann2022learning, lee2024active} ensure that the confidence set $\Phi$ is constructed such that $\phi_\star\in\Phi$ with probability at least $1-1/N$. The regret is decomposed into three parts: that of the initial exploration phase, that of the second phase under the failure event where $\phi_\star \notin \Phi$, and that of the second phase under the success event, where $\phi_\star\in\Phi$. Using the bound on the episode costs, the regret incurred from the first phase is bounded by $TL_{\mathsf{cost}} N_{\mathsf{phase\,1}}$ and the regret incurred during the second phase under the failure event $\phi_\star\notin\Phi$ is bounded by $TL_{\mathsf{cost}}(N-N_{\mathsf{phase\,1}})\mathbb{P}\brac{\phi_\star \notin\Phi} \leq L_{\mathsf{cost}}$. The condition on the radius of the confidence set ensures that the prediction error is strongly convex when the learner plays a certainty equivalent controller synthesized using any system estimate $\phi\in\Phi$. This leads to an analysis similar to that of stochastic gradient descent \citep{robbins1951sgd} on a strongly convex objective to bound the regret incurred during the second phase. Summing the contributions of the three components leads to the regret bound in \Cref{thm:LogRegret}. 
\end{proof}

Before proceeding, we note that while the regret of \Cref{alg:LogRegret} depends \textit{polylogarithmically} on the number of episodes $N$, it depends \textit{polynomially (quasilinearly when $\alpha = 1/2$)} on the episode length $T$. Intuitively, one might expect a sublinear dependence on $T$ since increasing $T$ increases the number of interactions the learner has with the system. The polynomial dependence on $T$ arises because we consider an episodic setting without mixing assumptions within episodes. As a result, the sensitivity of the episode cost to parameter estimation scales quadratically in the episode horizon. Imposing mixing assumptions can reduce this to linear scaling by modifying the proof of Lemma D.1 of \citet{wagenmaker2024optimal}. 
% Indeed, under such mixing assumptions, growing length of the episode reduce the identification error (see Appendix B.2 of \citep{ziemann2022single}). 
Therefore, by imposing stronger assumptions which lead to mixing, such as stability of the initial and optimal policies, one can likely achieve a logarithmic dependence on $T$. We leave formalizing this to future work.

\section{Numerical Validation} \label{s:numerical}

\subsection{Toy Experiment}
\label{toy example}

We provide an simple example to illustrate the fast regret rates attained by Algorithm \ref{alg:LogRegret}. For more experiments, see \cref{cartpole experiment}. Consider the two-dimensional nonlinear system\begin{align}\label{expr: toy system}
    x_{t + 1} = x_t + 5\exp\paren{-\norm{x_t - \phi^*}^2}\frac{x_t - \phi^*}{\norm{x_t - \phi^*}} \! +\! u_t \!+\! w_t 
\end{align}
where $x_t, u_t, w_t, \phi^* \in \mathbb{R}^2$, and with $x_1 = \begin{bmatrix}
    0 & 0
\end{bmatrix}^\top$. The noise $w_t$ has a standard normal distribution. We choose the unknown parameter $\phi^* = \begin{bmatrix}
    0.25 & 0.25
\end{bmatrix}^\top$.

In this experiment, we use the horizon $T = 10$ and the number of episodes $N = 3000$. We will consider the quadratic cost functions\begin{align*}
    c_t(x, u) &= \norm{x}^2 \quad\text{for $t = 1, \dots, T$}.
\end{align*}
The policy class $\Pi$ consists of controllers parameterized by the dynamics estimate $\hat\phi$, with\begin{align}\label{expr: toy policy class}
        \pi_{\hat\phi}(x) =  - \paren{x + 5 \exp\paren{-\norm{x_t - \hat\phi}^2}\frac{x_t - \hat\phi}{\norm{x_t - \hat\phi}}}.
\end{align}
It can be shown that the dynamics \eqref{expr: toy system} and policy class \eqref{expr: toy policy class} satisfy Assumption \ref{assume:PersistencyOfExcitationOptimalController}. Our initial policy $\pi_0$ plays the controller $\pi_{\phi}$ corresponding to $\phi = \begin{bmatrix}
    0 & 0
\end{bmatrix}^\top$, which can be shown to satisfy Assumption \ref{assume:InitialLojasiewiczPolicy}. In place of choosing $N_{\mathsf{phase\,1}}$ or $r_{\Phi}$ according to Theorem 1, we heuristically set $N_{\mathsf{phase\,1}} = 100$ and $r_{\Phi} = 0.2$. We note that the dynamics are not uniformly bounded globally, however they are uniformly bounded with high probability.

Under this choice of cost function and policy class, the learner's objective is to keep the system near the origin. Figure \ref{fig:regret_plots} illustrates the performance (measured in terms of regret) of Algorithm \ref{alg:LogRegret} on the toy dynamical system. The first plot shows that, after the initial $N_{\mathsf{phase\,1}}$-episode initial phase, the excess cost incurred per round begins to decay quickly, leading to the regret growing polylogarithmically with $N$. The second plot is included to better illustrate the regret attained by Algorithm \ref{alg:LogRegret}; after the initial phase, the average regret appears to grow as a polynomial of the logarithm of the iteration. This toy example highlights the fast regret rates attained by Algorithm \ref{alg:LogRegret}.

\begin{figure*}[h!]
    \centering
    \includegraphics[width=0.6\linewidth]{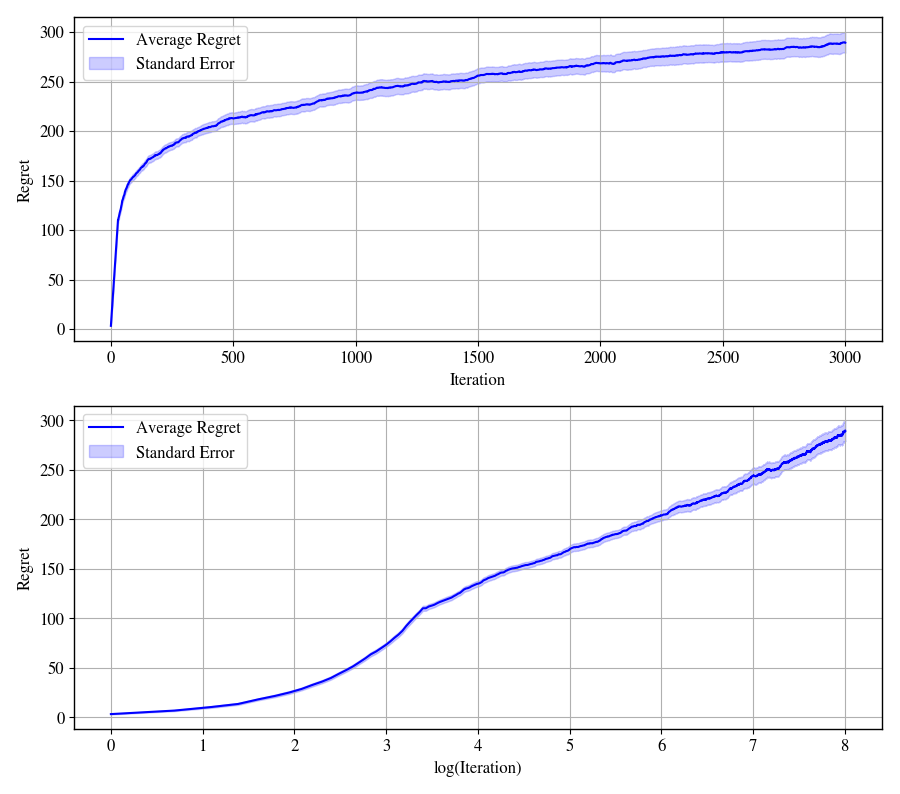}
    \caption{Average regret incurred by Algorithm \ref{alg:LogRegret} on the toy dynamical system \eqref{expr: toy system}, versus iterations and $\log(\text{iterations})$, respectively. The mean over 30 runs is shown, with the standard error shaded.}
    \label{fig:regret_plots}
\end{figure*}

\subsection{Cartpole Experiment}\label{cartpole experiment}

In this section, we complement our simple numerical example with an implementation of \Cref{alg:LogRegret} on a cartpole system defined by the dynamics:\begin{align}
(M + m) (\ddot{p} + b_p \dot{p}) + m l \cos(\theta) (\ddot{\theta} + b_\theta \dot{\theta}) &= m \ell \dot{\theta}^2 \sin(\theta) + u, \label{cartpole system 1} \\
m \cos(\theta) (\ddot{p} + b_p \dot{p}) + m l (\ddot{\theta} + b_\theta \dot{\theta}) &= m g \sin(\theta). \label{cartpole system 2}
\end{align}
Here, $p$ is the position of the cart, $\theta$ is the angle of the pole from the upright position, $u$ is the control force; the state vector is given by $x = \begin{bmatrix}
    p & \Dot p & \theta & \Dot \theta
\end{bmatrix}^\top$ and the input is given by $u$. Also, $M$ is the mass of the cart, $m$ is the mass of the pole, $l$ is the length of the pole, $g$ is the acceleration due to gravity, $b_x$ is the friction coefficient for the cart, and $b_\theta$ is the friction coefficient for the pole. We discretize the system using the Euler approach using a timestep of $dt = 0.2$. We also include additive zero mean Gaussian noise with covariance $0.05 I_4$. The unknown parameters are $\phi^* = \begin{bmatrix}M & m &l& b_x &b_\theta \end{bmatrix}^\top = \begin{bmatrix}1& 0.1& 1& 1& 1\end{bmatrix}^\top$. For every episode, the system starts from the upright position, given by the state $x_0 = \begin{bmatrix} 0&0&0&0\end{bmatrix}^\top$. The desired behavior is to keep the pole upright with the cart positioned at the origin for a time horizon of $T = 20$ timesteps. This behavior is described by the quadratic cost functions $c_t(x, u) = \norm{x}^2 + 0.1 u^2$.

Our exploitation policy class $\Pi$ is given by neural networks with layer sizes $(4, 64, 64, 64, 1)$ and ReLU activation functions. For computational reasons, in place of directly solving for the certainty equivalent policy for each parameter estimate $\phi_i$, we simultaneously update a dynamics estimate $\phi_i$ and train our control parameters $\theta_i$ as follows. At each iteration, we update our estimate of $\phi_i$ as in \Cref{alg:LogRegret} to get a new estimate $\phi_{i + 1}$; we then use the Adam optimizer \citep{Kingma2014AdamAM} to train a new set of control parameters $\theta_{i + 1}$ to minimize the cost functions using trajectories sampled with the dynamics $\phi_{i + 1}$ (in place of $\phi^*$), warm-starting the optimizer with the previous control parameters $\theta_i$. The initial exploration policy $\pi_0$ is given by bounded random noise scaled to match a predefined energy budget over the time horizon $T$; we choose a budget of $0.1T$. Finally, to illustrate the performance of our algorithm, we trained a "best-in-class" controller $\pi^*$ using trajectories sampled with the true dynamics $\phi^*$.

In this experiment, we use the horizon $T = 20$ and the number of episodes $N = 300$. Finally, we note that in place of choosing the number of initial phase episodes $N_{\mathsf{phase\,1}}$, the confidence radius $r_\Phi$, and the step sizes $\eta_i$ according to \Cref{cor: hyperparameter choice} and \Cref{alg:LogRegret}, we heuristically set $N_{\mathsf{phase\,1}} = 1$, $r_\Phi = 1$, and $\eta_i = 100 / (100 + i)$. The cost of each controller was evaluated by sampling $10000$ trajectories and using the average cost; for computational reasons, we chose to only evaluate the cost every $10$ iterations.

\begin{figure}
    \centering
    \includegraphics[width=0.6\linewidth]{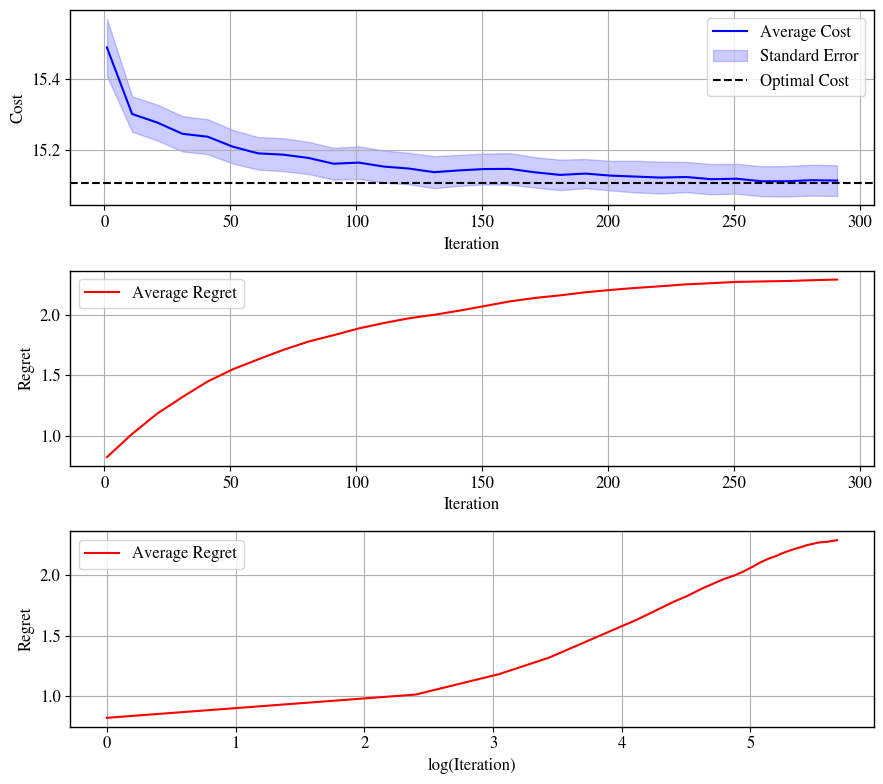}
    \caption{The first plot shows average cost incurred by \Cref{alg:LogRegret} on the cartpole system \eqref{cartpole system 1} - \eqref{cartpole system 2}, versus iterations. The mean over 30 runs is shown in blue, with standard error shaded. The cost of a "best-in-class" controller is shown with the dashed black line. The second and third plots show average regret versus iterations and the logarithm of iterations, respectively.}
    \label{fig: cartpole results}
\end{figure}

Figure \ref{fig: cartpole results} illustrates the cost incurred by Algorithm \ref{alg:LogRegret} on the cartpole system. The first plot shows that the cost of the controllers chosen by Algorithm \ref{alg:LogRegret} converges to the cost of $\pi^*$ quickly, which in turn leads to sublinear regret as demonstrated in the second plot. We note that unlike the example in \Cref{toy example}, the plot of regret versus logarithm of iteration does not show the same linear growth. This is due to two main reasons: first, we observed higher variance in estimating the costs of our cartpole controllers via sampling, leading to higher estimation error in both the average performance of \Cref{alg:LogRegret} as well as the optimal cost; second, due to the nonconvexity of optimizing neural network weights, our policy optimization steps were inexact, introducing additional discrepancies with our theory. However, the overall trend of fast convergence to the optimal control cost using a greedy algorithm is clear, and supports the behavior predicted by our theory. This cartpole experiment verifies that \Cref{alg:LogRegret} works on simple physical systems in practice.

\section{Conclusion}

We have introduced Algorithm \ref{alg:LogRegret} for online learning in a broad class of nonlinear dynamical systems. We have also proven a general sufficient condition for polylogarithmic regret under a natural curvature condition --- when the Fisher information matrix at the optimal policy is positive definite (detailed in our Assumption \ref{assume:PersistencyOfExcitationOptimalController}) --- and show that polylogarithmic regret is achieved by our Algorithm \ref{alg:LogRegret}. Finally, we have verified the performance of Algorithm \ref{alg:LogRegret} on a toy dynamical system and show that it achieves a fast regret rate in practice. Future work could extend these results to the single-trajectory setting. In particular, it could be interesting to extend the $\log^2 N$ regret rates of \cite{cassel2020logarithmic} and \cite{lee2024nonasymptotic} in the single-trajectory partially known linear setting to the setting with nonlinear dynamics. Another exciting avenue for future work is to design an online learning algorithm which deploys optimal experiment design techniques \citep{wagenmaker2024optimal} to optimally balance exploration and exploitation. Doing so may result in algorithms which automatically determine whether \Cref{assume:PersistencyOfExcitationOptimalController} is satisfied. Such an algorithm could achieve logarithmic regret if possible, and otherwise achieve $\sqrt{N}$ regret. Additionally, it may be possible to show improved dependence on the system-theoretic constants by using this approach.

\section*{Acknowledgements}

BL and NM are supported by NSF Award SLES-2331880, NSF CAREER award ECCS-2045834 and AFOSR Award FA9550-24-1-0102.
IZ is supported by a Swedish Research Council international postdoc grant. 

\bibliography{references.bib}

\begin{thebibliography}{44}
\providecommand{\natexlab}[1]{#1}
\providecommand{\url}[1]{\texttt{#1}}
\expandafter\ifx\csname urlstyle\endcsname\relax
  \providecommand{\doi}[1]{doi: #1}\else
  \providecommand{\doi}{doi: \begingroup \urlstyle{rm}\Url}\fi

\bibitem[Boffi et~al.(2021)Boffi, Tu, and Slotine]{boffi2021regret}
Nicholas~M Boffi, Stephen Tu, and Jean-Jacques~E Slotine.
\newblock Regret bounds for adaptive nonlinear control.
\newblock In \emph{Learning for Dynamics and Control}, pages 471--483. PMLR, 2021.

\bibitem[Burnetas and Katehakis(1997)]{burnetas1997optimal}
Apostolos~N Burnetas and Michael~N Katehakis.
\newblock Optimal adaptive policies for markov decision processes.
\newblock \emph{Mathematics of Operations Research}, 22\penalty0 (1):\penalty0 222--255, 1997.

\bibitem[Cassel et~al.(2020)Cassel, Cohen, and Koren]{cassel2020logarithmic}
Asaf Cassel, Alon Cohen, and Tomer Koren.
\newblock Logarithmic regret for learning linear quadratic regulators efficiently.
\newblock In \emph{International Conference on Machine Learning}, pages 1328--1337. PMLR, 2020.

\bibitem[Charlet et~al.(1989)Charlet, Levine, and Marino]{charlet1989dynamic}
B~Charlet, J~Levine, and R~Marino.
\newblock On dynamic feedback linearization.
\newblock \emph{Systems \& Control Letters}, 13\penalty0 (2):\penalty0 143--151, 1989.

\bibitem[Dann and Brunskill(2015)]{dann2015sample}
Christoph Dann and Emma Brunskill.
\newblock Sample complexity of episodic fixed-horizon reinforcement learning.
\newblock \emph{Advances in Neural Information Processing Systems}, 28, 2015.

\bibitem[Dean et~al.(2020)Dean, Mania, Matni, Recht, and Tu]{dean2020sample}
Sarah Dean, Horia Mania, Nikolai Matni, Benjamin Recht, and Stephen Tu.
\newblock On the sample complexity of the linear quadratic regulator.
\newblock \emph{Foundations of Computational Mathematics}, 20\penalty0 (4):\penalty0 633--679, 2020.

\bibitem[Faradonbeh et~al.(2018)Faradonbeh, Tewari, and Michailidis]{faradonbeh2018finite}
Mohamad Kazem~Shirani Faradonbeh, Ambuj Tewari, and George Michailidis.
\newblock Finite time identification in unstable linear systems.
\newblock \emph{Automatica}, 96:\penalty0 342--353, 2018.

\bibitem[Fazel et~al.(2018)Fazel, Ge, Kakade, and Mesbahi]{fazel2018global}
Maryam Fazel, Rong Ge, Sham Kakade, and Mehran Mesbahi.
\newblock Global convergence of policy gradient methods for the linear quadratic regulator.
\newblock In \emph{International conference on machine learning}, pages 1467--1476. PMLR, 2018.

\bibitem[Feldbaum(1960{\natexlab{a}})]{feldbaum1960dual1}
AA~Feldbaum.
\newblock {Dual Control Theory. I}.
\newblock \emph{Avtomatika i Telemekhanika}, 21\penalty0 (9):\penalty0 1240--1249, 1960{\natexlab{a}}.

\bibitem[Feldbaum(1960{\natexlab{b}})]{feldbaum1960dual2}
AA~Feldbaum.
\newblock {Dual Control Theory. II}.
\newblock \emph{Avtomatika i Telemekhanika}, 21\penalty0 (11):\penalty0 1453--1464, 1960{\natexlab{b}}.

\bibitem[Foster et~al.(2020)Foster, Sarkar, and Rakhlin]{foster2020learning}
Dylan Foster, Tuhin Sarkar, and Alexander Rakhlin.
\newblock Learning nonlinear dynamical systems from a single trajectory.
\newblock In \emph{Learning for Dynamics and Control}, pages 851--861. PMLR, 2020.

\bibitem[Garivier et~al.(2019)Garivier, M{\'e}nard, and Stoltz]{garivier2019explore}
Aur{\'e}lien Garivier, Pierre M{\'e}nard, and Gilles Stoltz.
\newblock Explore first, exploit next: The true shape of regret in bandit problems.
\newblock \emph{Mathematics of Operations Research}, 44\penalty0 (2):\penalty0 377--399, 2019.

\bibitem[Jedra and Proutiere(2022)]{jedra2022minimal}
Yassir Jedra and Alexandre Proutiere.
\newblock Minimal expected regret in linear quadratic control.
\newblock In \emph{International Conference on Artificial Intelligence and Statistics}, pages 10234--10321. PMLR, 2022.

\bibitem[Jiang et~al.(2017)Jiang, Krishnamurthy, Agarwal, Langford, and Schapire]{jiang2017contextual}
Nan Jiang, Akshay Krishnamurthy, Alekh Agarwal, John Langford, and Robert~E. Schapire.
\newblock Contextual decision processes with low {B}ellman rank are {PAC}-learnable.
\newblock In Doina Precup and Yee~Whye Teh, editors, \emph{Proceedings of the 34th International Conference on Machine Learning}, volume~70 of \emph{Proceedings of Machine Learning Research}, pages 1704--1713. PMLR, 06--11 Aug 2017.
\newblock URL \url{https://proceedings.mlr.press/v70/jiang17c.html}.

\bibitem[Kakade et~al.(2020)Kakade, Krishnamurthy, Lowrey, Ohnishi, and Sun]{kakade2020information}
Sham Kakade, Akshay Krishnamurthy, Kendall Lowrey, Motoya Ohnishi, and Wen Sun.
\newblock Information theoretic regret bounds for online nonlinear control.
\newblock \emph{Advances in Neural Information Processing Systems}, 33:\penalty0 15312--15325, 2020.

\bibitem[Kingma and Ba(2014)]{Kingma2014AdamAM}
Diederik~P. Kingma and Jimmy Ba.
\newblock Adam: A method for stochastic optimization.
\newblock \emph{CoRR}, abs/1412.6980, 2014.
\newblock URL \url{https://api.semanticscholar.org/CorpusID:6628106}.

\bibitem[Kober et~al.(2013)Kober, Bagnell, and Peters]{kober2013reinforcement}
Jens Kober, J~Andrew Bagnell, and Jan Peters.
\newblock Reinforcement learning in robotics: A survey.
\newblock \emph{The International Journal of Robotics Research}, 32\penalty0 (11):\penalty0 1238--1274, 2013.

\bibitem[Lai(1986)]{lai1986asymptotically}
Tze~Leung Lai.
\newblock Asymptotically efficient adaptive control in stochastic regression models.
\newblock \emph{Advances in Applied Mathematics}, 7\penalty0 (1):\penalty0 23--45, 1986.

\bibitem[Lai and Robbins(1985)]{lai1985asymptotically}
Tze~Leung Lai and Herbert Robbins.
\newblock Asymptotically efficient adaptive allocation rules.
\newblock \emph{Advances in applied mathematics}, 6\penalty0 (1):\penalty0 4--22, 1985.

\bibitem[Lale et~al.(2020)Lale, Azizzadenesheli, Hassibi, and Anandkumar]{lale2020logarithmic}
Sahin Lale, Kamyar Azizzadenesheli, Babak Hassibi, and Anima Anandkumar.
\newblock Logarithmic regret bound in partially observable linear dynamical systems.
\newblock \emph{Advances in Neural Information Processing Systems}, 33:\penalty0 20876--20888, 2020.

\bibitem[Lale et~al.(2024)Lale, Renn, Azizzadenesheli, Hassibi, Gharib, and Anandkumar]{lale2024falcon}
Sahin Lale, Peter~I Renn, Kamyar Azizzadenesheli, Babak Hassibi, Morteza Gharib, and Anima Anandkumar.
\newblock {FALCON}: Fourier adaptive learning and control for disturbance rejection under extreme turbulence.
\newblock \emph{npj Robot}, 2\penalty0 (1), September 2024.

\bibitem[Lee et~al.(2024{\natexlab{a}})Lee, Rantzer, and Matni]{lee2024nonasymptotic}
Bruce Lee, Anders Rantzer, and Nikolai Matni.
\newblock Nonasymptotic regret analysis of adaptive linear quadratic control with model misspecification.
\newblock In \emph{6th Annual Learning for Dynamics \& Control Conference}, pages 980--992. PMLR, 2024{\natexlab{a}}.

\bibitem[Lee et~al.(2024{\natexlab{b}})Lee, Ziemann, Pappas, and Matni]{lee2024active}
Bruce~D Lee, Ingvar Ziemann, George~J Pappas, and Nikolai Matni.
\newblock Active learning for control-oriented identification of nonlinear systems.
\newblock \emph{arXiv preprint arXiv:2404.09030}, 2024{\natexlab{b}}.

\bibitem[Lin et~al.(1985)Lin, Kumar, and Seidman]{lin1985will}
Woei Lin, PR~Kumar, and TI~Seidman.
\newblock Will the self-tuning approach work for general cost criteria?
\newblock \emph{Systems \& control letters}, 6\penalty0 (2):\penalty0 77--85, 1985.

\bibitem[Ljung(1999)]{ljung1999system}
Lennart Ljung.
\newblock \emph{System identification: theory for the user}.
\newblock PTR Prentice Hall, Upper Saddle River, NJ, 1999.

\bibitem[Mania et~al.(2022)Mania, Jordan, and Recht]{mania2022active}
Horia Mania, Michael~I Jordan, and Benjamin Recht.
\newblock Active learning for nonlinear system identification with guarantees.
\newblock \emph{Journal of Machine Learning Research}, 23\penalty0 (32):\penalty0 1--30, 2022.

\bibitem[Musavi et~al.(2024)Musavi, Guo, Dullerud, and Li]{musavi2024identification}
Negin Musavi, Ziyao Guo, Geir Dullerud, and Yingying Li.
\newblock Identification of analytic nonlinear dynamical systems with non-asymptotic guarantees.
\newblock \emph{arXiv preprint arXiv:2411.00656}, 2024.

\bibitem[Ok et~al.(2018)Ok, Proutiere, and Tranos]{ok2018exploration}
Jungseul Ok, Alexandre Proutiere, and Damianos Tranos.
\newblock Exploration in structured reinforcement learning.
\newblock \emph{Advances in Neural Information Processing Systems}, 31, 2018.

\bibitem[Ouyang et~al.(2022)Ouyang, Wu, Jiang, Almeida, Wainwright, Mishkin, Zhang, Agarwal, Slama, Ray, et~al.]{ouyang2022training}
Long Ouyang, Jeffrey Wu, Xu~Jiang, Diogo Almeida, Carroll Wainwright, Pamela Mishkin, Chong Zhang, Sandhini Agarwal, Katarina Slama, Alex Ray, et~al.
\newblock Training language models to follow instructions with human feedback.
\newblock \emph{Advances in neural information processing systems}, 35:\penalty0 27730--27744, 2022.

\bibitem[Polderman(1986)]{polderman1986necessity}
Jan~Willem Polderman.
\newblock On the necessity of identifying the true parameter in adaptive lq control.
\newblock \emph{Systems \& control letters}, 8\penalty0 (2):\penalty0 87--91, 1986.

\bibitem[Robbins and Monro(1951)]{robbins1951sgd}
Herbert Robbins and Sutton Monro.
\newblock {A Stochastic Approximation Method}.
\newblock \emph{The Annals of Mathematical Statistics}, 22\penalty0 (3):\penalty0 400 -- 407, 1951.
\newblock \doi{10.1214/aoms/1177729586}.
\newblock URL \url{https://doi.org/10.1214/aoms/1177729586}.

\bibitem[Roulet and d'Aspremont(2017)]{roulet2017sharpness}
Vincent Roulet and Alexandre d'Aspremont.
\newblock Sharpness, restart and acceleration.
\newblock \emph{Advances in Neural Information Processing Systems}, 30, 2017.

\bibitem[Sattar and Oymak(2022)]{sattar2022non}
Yahya Sattar and Samet Oymak.
\newblock Non-asymptotic and accurate learning of nonlinear dynamical systems.
\newblock \emph{Journal of Machine Learning Research}, 23\penalty0 (140):\penalty0 1--49, 2022.

\bibitem[Silver et~al.(2017)Silver, Schrittwieser, Simonyan, Antonoglou, Huang, Guez, Hubert, Baker, Lai, Bolton, et~al.]{silver2017mastering}
David Silver, Julian Schrittwieser, Karen Simonyan, Ioannis Antonoglou, Aja Huang, Arthur Guez, Thomas Hubert, Lucas Baker, Matthew Lai, Adrian Bolton, et~al.
\newblock Mastering the game of go without human knowledge.
\newblock \emph{nature}, 550\penalty0 (7676):\penalty0 354--359, 2017.

\bibitem[Simchowitz and Foster(2020)]{simchowitz2020naive}
Max Simchowitz and Dylan Foster.
\newblock Naive exploration is optimal for online lqr.
\newblock In \emph{International Conference on Machine Learning}, pages 8937--8948. PMLR, 2020.

\bibitem[Simchowitz et~al.(2018)Simchowitz, Mania, Tu, Jordan, and Recht]{simchowitz2018learning}
Max Simchowitz, Horia Mania, Stephen Tu, Michael~I Jordan, and Benjamin Recht.
\newblock Learning without mixing: Towards a sharp analysis of linear system identification.
\newblock In \emph{Conference On Learning Theory}, pages 439--473. PMLR, 2018.

\bibitem[Treven et~al.(2023)Treven, Sancaktar, Blaes, Coros, and Krause]{treven2023optimistic}
Lenart Treven, Cansu Sancaktar, Sebastian Blaes, Stelian Coros, and Andreas Krause.
\newblock Optimistic active exploration of dynamical systems.
\newblock \emph{Advances in Neural Information Processing Systems}, 36:\penalty0 38122--38153, 2023.

\bibitem[Wagenmaker et~al.(2024)Wagenmaker, Shi, and Jamieson]{wagenmaker2024optimal}
Andrew Wagenmaker, Guanya Shi, and Kevin~G Jamieson.
\newblock Optimal exploration for model-based rl in nonlinear systems.
\newblock \emph{Advances in Neural Information Processing Systems}, 36, 2024.

\bibitem[Xu et~al.(2021)Xu, Ma, and Du]{xu2021fine}
Haike Xu, Tengyu Ma, and Simon Du.
\newblock Fine-grained gap-dependent bounds for tabular mdps via adaptive multi-step bootstrap.
\newblock In \emph{Conference on Learning Theory}, pages 4438--4472. PMLR, 2021.

\bibitem[Yang et~al.(2020)Yang, Caluwaerts, Iscen, Zhang, Tan, and Sindhwani]{yang2020data}
Yuxiang Yang, Ken Caluwaerts, Atil Iscen, Tingnan Zhang, Jie Tan, and Vikas Sindhwani.
\newblock Data efficient reinforcement learning for legged robots.
\newblock In \emph{Conference on Robot Learning}, pages 1--10. PMLR, 2020.

\bibitem[Zhou et~al.(2021)Zhou, Gu, and Szepesvari]{zhou2021nearly}
Dongruo Zhou, Quanquan Gu, and Csaba Szepesvari.
\newblock Nearly minimax optimal reinforcement learning for linear mixture markov decision processes.
\newblock In \emph{Conference on Learning Theory}, pages 4532--4576. PMLR, 2021.

\bibitem[Ziemann and Sandberg(2024)]{ziemann2024regret}
Ingvar Ziemann and Henrik Sandberg.
\newblock Regret lower bounds for learning linear quadratic gaussian systems.
\newblock \emph{IEEE Transactions on Automatic Control}, 2024.

\bibitem[Ziemann and Tu(2022)]{ziemann2022learning}
Ingvar Ziemann and Stephen Tu.
\newblock Learning with little mixing.
\newblock \emph{Advances in Neural Information Processing Systems}, 35:\penalty0 4626--4637, 2022.

\bibitem[Ziemann et~al.(2024)Ziemann, Tu, Pappas, and Matni]{ziemann2024sharp}
Ingvar Ziemann, Stephen Tu, George~J Pappas, and Nikolai Matni.
\newblock Sharp rates in dependent learning theory: Avoiding sample size deflation for the square loss.
\newblock In \emph{Forty-first International Conference on Machine Learning}, 2024.

\end{thebibliography}

\appendix

\onecolumn
% \section{RESULTS FOR THE HARD-TO-IDENTIFY CASE} \label{s: hard to identify}
% \label{s:explorethencommitproof}

\section{Slow Learning}
\label{s: slow learning}
% \label{s:explorethencommitbound}

Here, we consider a more general setting where the assumption that the optimal policy has a positive definite Fisher Information (\Cref{assume:PersistencyOfExcitationOptimalController}) does not necessarily hold. In this setting, the approach of \Cref{s: fast learning} fails, because the prediction error is no longer necessarily locally strongly convex near the optimal solution. Consequently, the result of the previous section no longer applies to achieve polylogarithmic regret. Instead, we propose an algorithm for which the learner incurs regret scaling with the square root of the number of interactions. 

\subsection{Additional Assumptions}

 To present the result for the setting where \Cref{assume:PersistencyOfExcitationOptimalController} does not necessarily hold, we strengthen the condition on the initial policy from that in \Cref{assume:InitialLojasiewiczPolicy} to the following. 
\begin{assumption} \label{assume:GoodInitialLojasiewiczPolicy}
    (Initial Lojasiewicz policy). Fix some positive constant $C_{\mathsf{Loja}}$. The learner has access to a policy $\pi_0 \in \Pi$ which is $(C_{\mathsf{Loja}}, 1/2)$-Lojasiewicz.
\end{assumption}
In particular, we restrict attention to settings where the Lojasiewicz condition holds with parameter $1/2$. This means that estimation error grows quadratically with the parameter error. The assumption is made for ease of exposition. It would instead suffice to keep \Cref{assume:InitialLojasiewiczPolicy} and additionally assume that there exists a policy in the policy class which has a positive definite Fisher Information\footnote{Note that this is still substantially less restrictive than assuming that the optimal policy satisfies such a condition.}. Then by using the optimal experiment design procedure of \citet{lee2024active}, one could find the policy with a positive definite Fisher information, which satisfies the Lojasiewicz condition with parameter $\alpha=1/2$. 

\subsection{Algorithm and Regret Bound}

Under Assumptions \ref{assume:SmoothDynamics}, \ref{assume:SmoothPolicyClass}, \ref{assume:BoundedCosts}, \ref{assume:WagenmakerProp6}, and \ref{assume:GoodInitialLojasiewiczPolicy}, we give an algorithm with $\mathbb E[\Regret(N)] = O(\sqrt N \log N)$. This algorithm is based on an "explore then commit" procedure in which the algorithm explores for some number of episodes to collect an initial dataset to synthesize a control policy $\hat \pi$, then plays $\hat \pi$ for the remainder of the episodes. 

\begin{algorithm}
\caption{Explore-Then-Commit}\label{alg:SqrtRegret}
\begin{algorithmic}[1]
\Require Initial policy $\pi_0$, policy class $\Pi$, number of episodes $N$
\State Play $\pi_0$ for $\sqrt N$ episodes to collect the dataset $\calD \gets \curly{(x_t, u_t, x_{t+1}}_{t=1, n=1}^{T,\sqrt{N}}$.
\State \label{line: ls estimate ec} Set $\hat \phi$ as the least squares estimate \eqref{expr:NonlinearLeastSquares} using the dataset $\calD$. 
\State Set $\hat \pi \gets \pi^\star(\hat\phi)$.
\State Play $\hat \pi$ for the remaining $N-\sqrt{N}$ episodes. 
\end{algorithmic}
\end{algorithm}

\Cref{alg:SqrtRegret} is simpler than \Cref{alg:LogRegret}, and just includes a single step of parameter estimation rather than continuously refining. We characterize the regret incurred by Algorithm \ref{alg:SqrtRegret} as follows.

\begin{theorem}\label{thm:SqrtRegret}
    Consider applying \Cref{alg:SqrtRegret} to the system \eqref{expr:Dynamics} with initial policy $\pi_0$ satisfying \Cref{assume:GoodInitialLojasiewiczPolicy}, policy class $\Pi$ satisfying \Cref{assume:SmoothPolicyClass} for $N$ episodes. Suppose that the dynamics satisfy \Cref{assume:SmoothDynamics} and that the costs satisfy \Cref{assume:BoundedCosts}. Furthermore, suppose that the dynamics, objective, and policy class satisfy \Cref{assume:WagenmakerProp6}. Then there exists polynomial function $\mathsf{poly}_{\mathsf{burn}\textnormal{ }\mathsf{in}}$ such that
    \begin{align*}
        \mathbb{E}&\brac{\Regret(N)} \leq \mathsf{poly}\paren{L_{\pi^\star},  L_f, \log B,  L_{\Pi}, L_{\mathsf{cost}}, \sigma, \sigma^{-1}, T, d_x, d_{\phi}, \log N} \sqrt{N}.
    \end{align*}
    as long as $N \geq \mathsf{poly}_{\mathsf{burn\,in}}\paren{\sigma, d_x, d_{\phi}, \log B, L_f, \log (\sigma^{-1}), \log N, C_{\mathsf{Loja}}, {r_{\mathsf{ce}}}^{-1}, {r_{\mathsf{cost}}}^{-1}}$.
\end{theorem}

This result follows by noting that the exploration phase incurs a regret proportional to the number of exploration episodes, $\sqrt{N}$. For the exploitation phase, we leverage the smoothness of the cost functions, policy classes, and dynamics, to show that the excess cost incurred from a certainty equivalent policy scales quadratically with the estimation error $\norm{\hat\phi-\phi_\star}$. Meanwhile, the system identification bounds of \citet{lee2024active} demonstrate that the estimation error decays with $1/\sqrt{K}$ when $K$ trajectories are used to fit $\hat\phi$ (neglecting system constants and log terms). By setting $K=\sqrt{N}$, we find that $\norm{\hat\phi-\phi_\star}$ scales with $1/N^{1/4}$. Then the regret incurred in the exploration phase is $N \times \norm{\hat\phi-\phi_\star}^2 \lesssim N \times \frac{1}{\sqrt{N}} = \sqrt{N}$, where $\lesssim$ hides logarithmic quantities and problem constants. Summing the regrets in the exploration phase with the exploitation phase then also results in a bound scaling with $\sqrt{N}$. See \Cref{s:explorethencommitproof} for a rigorous proof.

\subsection{Proof of \Cref{thm:SqrtRegret}} \label{s:explorethencommitproof}

To prove Theorem \ref{thm:SqrtRegret} we first state two lemmas from \citet{lee2024active}). 

\begin{restatable}[Lemma A.1 of \citet{lee2024active}]{lemma}{consistency}
\label{lem: er bound}
Suppose Assumption~\ref{assume:SmoothDynamics} holds, and let $\delta\in(0,1/2]$. Let $\hat\phi$ be the least squares estimate from \eqref{expr:NonlinearLeastSquares} using $K$ episodes of data collected with a $(C_{\mathsf{Loja}}, \alpha)$ policy. There exists a polynomial function $\mathsf{poly}_{\alpha}$ which depends on $\alpha$ such that with probability at least $1-\delta$,
    \begin{align*}
        \norm{\hat\phi-\phi_\star}^2 \leq \left(\frac{512 \sigma^2}{TK } \paren{d_x + d_{\phi}\log\paren{\frac{4 B L_f TK}{\sigma \delta}}}\right)^{\alpha}.
    \end{align*}
    as long as $K \geq \tau_{\mathsf{err}}(\delta) \triangleq \mathsf{poly}_{\alpha}(\sigma,d_x, d_{\phi}, \log(B), L_f, C_{\mathsf{Loja}}, \log(K), \log\frac{1}{\delta})$.
\end{restatable}

\begin{lemma}[Modified from Lemma 3.1 of \citet{lee2024active}]
    \label{lem: cost decomposition}
    Suppose Assumptions~\ref{assume:SmoothDynamics}-\ref{assume:WagenmakerProp6} hold. 
    Then for $\hat\phi \in \calB(\phi^\star, \min\curly{r_{\mathsf{cost}}, r_{\mathsf{ce}}})$, 
    \begin{align}
            \label{eq: excess cost as parameter error}
            \calJ(\pi^\star(\hat\phi), \phi^\star)\! -\! \calJ(\pi^\star(\phi^\star), \phi^\star)\! \leq C_{\mathsf{cost}} \!\norm{\hat \phi \!-\! \phi^\star}^2,
        \end{align}
        where $C_{\mathsf{cost}} = \mathsf{poly}(L_{\pi^\star},  L_f, L_{\Pi}, L_{\mathsf{cost}}, \sigma^{-1}, d_x) T^2$.
\end{lemma} 
\begin{proof}
    The proof follows as in the proof of Lemma 3.1 of \citet{lee2024active}; however, the third order Taylor expansion is replaced with a second order expansion. 
\end{proof}

Armed with these results, we proceed to prove \Cref{thm:SqrtRegret}. To begin, decompose the regret into the event conditioned on the success event of \Cref{lem: er bound} that holds with probability at least $1-1/N$, and the regret conditioned on the complement of that event. The condition to apply \Cref{lem: er bound} is satisfied due to the burn-in condition on $N$. Denote the success event $\calE_{\mathsf{success}}$. It holds that 
    \begin{align*}
        \mathbb{E}\brac{\Regret(N)} &= \mathbb{P}(\calE_{\mathsf{success}}) \mathbb{E}\brac{\Regret(N) \mid \calE_{\mathsf{success}}}  + (1- \mathbb P (\calE_{\mathsf{success}})) \mathbb{E}\brac{\Regret(N) \mid \calE_{\mathsf{success}}^c} \\
        &\leq \mathbb{E}\brac{\Regret(N) \mid \calE_{\mathsf{success}}} + (1/N) N T L_{\mathsf{cost}} \\
        &= \mathbb{E}\brac{\sum_{n=1}^{N}{(J(\pi_n, \phi^*)- \min_{\pi \in \Pi}{J(\pi, \phi^*))}} \mid \calE_{\mathsf{success}}} + T L_{\mathsf{cost}},
    \end{align*}
    where the inequality follows from the probability bound on the failure event, and the fact that the cost incurred during each episode is bounded by $T L_{\mathsf{cost}}$. 
    We may further decompose the cost into the cost incurred during the exploration phase (episodes $1$ to $\sqrt{N}$) and the cost incurred in the exploitation phase $\sqrt{N}$ to $N$. Leveraging the fact that the costs are bounded, we may therefore bound the cost incurred during the exploration phase by $\sqrt{N} T L_{\mathsf{cost}}$: 
    \begin{align*}
         \mathbb{E}\brac{\Regret(N)} &\leq  TL_{\mathsf{cost}} + \mathbb{E}\brac{\sum_{n=1}^{\sqrt{N}}{(J(\pi_0, \phi^*)- \min_{\pi \in \Pi}{J(\pi, \phi^*))}} \mid \calE_{\mathsf{success}}} \\ 
         &\qquad + \mathbb{E}\brac{\sum_{n=\sqrt{N}+1}^{N}{(J(\pi^\star(\hat\phi), \phi^*)- \min_{\pi \in \Pi}{J(\pi, \phi^*))}} \mid \calE_{\mathsf{success}}} \\
         &= TL_{\mathsf{cost}} + \sqrt{N} \mathbb{E}\brac{{(J(\pi_0, \phi^*)- \min_{\pi \in \Pi}{J(\pi, \phi^*))}} \mid \calE_{\mathsf{success}}} \\
         & \qquad + (N-\sqrt{N})\mathbb{E}\brac{{(J(\pi^\star(\hat\phi), \phi^*)- \min_{\pi \in \Pi}{J(\pi, \phi^*))}} \mid \calE_{\mathsf{success}}}  \\
         &\leq  TL_{\mathsf{cost}} + \sqrt{N} L_{\mathsf{cost}} + N \mathbb{E}\brac{{(J(\pi^\star(\hat\phi), \phi^*)- \min_{\pi \in \Pi}{J(\pi, \phi^*))}} \mid \calE_{\mathsf{success}}}. 
    \end{align*}
    where $\hat\phi$ is the parameter estimate from \Cref{line: ls estimate ec} of \Cref{alg:SqrtRegret}. To continue the proof, we need to bound the term:

\[
N \mathbb{E}\left[ J(\pi^\star(\hat\phi), \phi^*) - \min_{\pi \in \Pi} J(\pi, \phi^*) \mid \mathcal{E}_{\mathsf{success}} \right],
\]
where \(\mathcal{E}_{\mathsf{success}}\) denotes the event where the estimation error \(\|\hat\phi - \phi^*\|\) is small, as guaranteed by Lemma \ref{lem: er bound}. Under the event \(\mathcal{E}_{\mathsf{success}}\), from Lemma \ref{lem: er bound}, it holds that
\[
\|\hat\phi - \phi^*\|^2 \leq \frac{512 \sigma^2}{T \sqrt{N}} \left( d_x + d_{\phi} \log\left( \frac{4 B L_f T \sqrt{N}}{\sigma \delta} \right) \right) \leq  \frac{C_{\mathsf{identification}}}{\sqrt{N}},
\]
where \(C_{\mathsf{identification}}  := \mathsf{poly}(\sigma, d_x, d_{\phi}, \log(B), \log(L_f), \log(T), \log(N))\). From Lemma \ref{lem: cost decomposition}, it holds that for \(\hat\phi\) within a neighborhood of \(\phi^*\),
\begin{align*}
J(\pi^\star(\hat\phi), \phi^*) - J(\pi^\star(\phi^*), \phi^*) &\leq C_{\mathsf{cost}} \|\hat\phi - \phi^*\|^2.
\end{align*}
In particular, it suffices to use the bound on $\norm{\hat\phi-\phi^\star}$ along with the burn-in condition on $N$ to ensure that $\norm{\hat\phi-\phi^\star}$ is small enough to instantiate \Cref{lem: cost decomposition}. Substituting this bound into the overall regret bound, it holds that 
\begin{align*}
\mathbb{E}\left[ \Regret(N) \right] &\leq TL_{\mathsf{cost}}(1+ \sqrt{N}) + N \mathbb{E}\left[ J(\pi^\star(\hat\phi), \phi^*) - J(\pi^\star(\phi^*), \phi^*) \mid \mathcal{E}_{\mathsf{success}} \right] \\
&\leq TL_{\mathsf{cost}} (1+\sqrt{N}) + C_{\mathsf{identification}} C_{\mathsf{cost}} \sqrt{N}.
\end{align*}
Thus, we conclude:
\[
\mathbb{E}\left[ \Regret(N) \right] \leq  \mathsf{poly}(L_{\pi^\star},  L_f, \log(B),  L_{\Pi}, L_{\mathsf{cost}}, \sigma, \sigma^{-1}, T, d_x, d_{\phi}, \log(N)) \sqrt{N}.
\]

\section{Proofs for Fast Learning}\label{appendix: B}

\subsection{Results for Persistently Exciting Systems}

Before proving Theorem \ref{thm:LogRegret}, we state and prove some useful results for systems for which the optimal controller is persistently exciting. First, we prove a technical lemma which will help us bound the smoothness of the prediction error with respect to the dynamics and the controller parameters.

\begin{lemma}\label{lemma: adapted random walk}
    Let $w_1, \dots, w_T \sim \mathcal N(0, \Sigma)$ be independent random vectors. Let $A_1, \dots, A_T \in \mathbb R^{m \times n}$ be random matrices adapted to the filtration $\mathcal F_t = \sigma(w_s \mid s \leq t - 1)$ such that $\norm{A_k}_{\mathsf{op}} \leq U$ for all $k$ with probability $1$. Define the sum $S \triangleq \sum_{t=1}^{T}{A_t w_t} \in \mathbb R^m$. Then, \begin{align*}
        \mathbb E\brac{\norm{S}^2} \leq U^2T \tr(\Sigma) \quad\text{and} \quad \mathbb E\brac{\norm{S}} \leq U\sqrt{T \tr(\Sigma)}.
    \end{align*}
\end{lemma}

\begin{proof}
    First, we aim to bound $\mathbb E\brac{\norm{S}^2}$. Expanding it out,\begin{align}\label{expected S2}
        \mathbb E\brac{\norm{S}^2} &= \mathbb E\brac{ \sum_{t=1}^{T}{\sum_{s = 1}^{T}{ w_t^\top A_t^\top A_s w_s }}  } = \mathbb E\brac{\sum_{t=1}^{T}{\norm{A_tw_t}^2}} + \mathbb E\brac{ \sum_{(s, t) \in [T]^2, s \neq t}{w_t^\top A_t^\top A_s w_s} }
    \end{align}
    By our assumptions, the second term on the right is zero, since\begin{align*}
        \mathbb E\brac{ \sum_{(s, t) \in [T]^2, s \neq t}{w_t^\top A_t^\top A_s w_s} } &= 2 \sum_{(s, t) \in [T]^2, s < t}{\mathbb E\brac{w_t^\top A_t^\top A_s w_s} } \tag{linearity} \\
        &= 2 \sum_{(s, t) \in [T]^2, s < t}{\mathbb E\brac{w_t^\top} \mathbb E\brac{A_t^\top A_s w_s} } \tag{$w_t$ is independent of $\mathcal F_{t - 1}$, and $A_t, A_s, w_s$ are $\mathcal F_{t - 1}$-measurable}\\
        &= 0 \tag{$w_t$ has zero mean}
    \end{align*}
    The remaining term from \eqref{expected S2} can be bounded as \begin{align}
        \mathbb E\brac{\sum_{t=1}^{T}{\norm{A_tw_t}^2}} \leq U^2 T \mathbb E\brac{\sum_{t=1}^{T}{\norm{w_t^2}} } = U^2 T \tr(\Sigma)
    \end{align}
    The desired result follows because $\mathbb E\brac{\norm{S}} \leq \sqrt{\mathbb E\brac{\norm{S}}^2}$ (an application of Jensen's inequality).
\end{proof}

Next, the following smoothness results show that if the optimal controller $\pi_{\theta^*(\phi^*)}$ is persistently exciting (Assumption \ref{assume:PersistencyOfExcitationOptimalController}), then all controllers $\pi_\theta$, where $\theta$ is in an open ball centered around $\theta^*$, are also persistently exciting. 

\begin{lemma}\label{lem:ThetaLipschitzFisherInformation}
    Suppose that Assumptions \ref{assume:SmoothDynamics} and \ref{assume:SmoothPolicyClass} hold. Then for all $\theta \in \mathbb{R}^{d_\theta}$,\begin{align*}
        \norm{ D_\theta \paren{D_{\phi}^{(2)} \Err_{\pi_\theta}^{\phi^*}(\phi^*)} \bigg\rvert_{\theta = \theta} }_{\mathsf{op}} \leq 2L_f^2L_{\Pi} + \frac{L_f^3 L_\Pi \sqrt{T d_x}}{\sigma}.
    \end{align*}
\end{lemma}

\begin{proof} Since the dynamics are rolled out under $\pi_\theta$, we have:\begin{dmath}\label{expr:ThetaDerivativeFisherInfo}
   D_\theta \paren{D_{\phi}^{(2)} \Err_{\pi_\theta}^{\phi^*}(\phi^*)} \bigg\rvert_{\theta = \theta}
   = D_\theta \paren{\mathbb E^{\phi^*}_{\pi_\theta} \brac{\frac{1}{T} \sum_{t = 1}^{T}{Df(x_t, \pi_\theta(x_t), \phi^*)^\top Df(x_t, \pi_\theta(x_t), \phi^*)} }  } \bigg\rvert_{\theta = \theta} \\
   = D_\theta \paren{ \int {\paren{ \frac{1}{T}\sum_{t = 1}^{T}{Df(x_t, \pi_\theta(x_t), \phi^*)^\top Df(x_t, \pi_\theta(x_t), \phi^*)} }\: p_{\pi_\theta}^{\phi^*}(x_{1:T+1})d(x_{1:T+1)}} } \bigg\rvert_{\theta = \theta}
\end{dmath}
Here $p_{\pi_\theta}^{\phi^*}(x_{1:T+1})$ is the density of the trajectory $x_{1:T+1}$ under dynamics $\phi$* and controller $\pi_\theta$:\begin{align}\label{expr:TrajectoryDensity}
    p_{\pi_\theta}^{\phi^*}(x_{1:T+1}) = \paren{\frac{1}{\sigma \sqrt {2\pi}}}^T \exp \paren{-\frac{1}{2\sigma^2}\sum_{t=1}^{T}{\| w_t \|^2}}
\end{align}
where $w_t \triangleq x_{t + 1} - f(x_t, \pi_{\theta}(x_t), \phi^*)$ is the noise. By the dominated convergence theorem, followed by the product rule:\begin{dmath}\label{expr:ThetaDerivativeFI}
    {\eqref{expr:ThetaDerivativeFisherInfo} = \int {\paren {D_\theta \paren{ \paren{\frac{1}{T} \sum_{t = 1}^{T}{Df(x_t, \pi_\theta(x_t), \phi^*)^\top Df(x_t, \pi_\theta(x_t), \phi^*)} } p_{\pi_\theta}^{\phi^*}(x_{1:T+1})} \bigg\rvert_{\theta = \theta} } \: d(x_{1:T+1)}}}  
    = \int { \paren{D_\theta \paren{\frac{1}{T}\sum_{t = 1}^{T}{Df(x_t, \pi_\theta(x_t), \phi^*)^\top Df(x_t, \pi_\theta(x_t), \phi^*)} } \bigg\rvert_{\theta = \theta} } \: p_{\pi_\theta}^{\phi^*}(x_{1:T+1}) d(x_{1:T+1})}  + \int {\paren{\frac{1}{T}\sum_{t = 1}^{T}{Df(x_t, \pi_\theta(x_t), \phi^*)^\top Df(x_t, \pi_\theta(x_t), \phi^*)}} \paren{D_\theta \paren{p_{\pi_\theta}^{\phi^*}(x_{1:T+1})I } \bigg\rvert_{\theta = \theta} } \: d(x_{1:T+1})}
\end{dmath}
We proceed by bounding the operator norms of the two integrals in equation \eqref{expr:ThetaDerivativeFI}. To bound the operator norm of the first integral, we apply Jensen's inequality followed by Assumptions \ref{assume:SmoothDynamics} and \ref{assume:SmoothPolicyClass}:\begin{dmath}
     \norm {\int { \paren{D_\theta \paren{\frac{1}{T}\sum_{t = 1}^{T}{Df(x_t, \pi_\theta(x_t), \phi^*)^\top Df(x_t, \pi_\theta(x_t), \phi^*)} } \bigg\rvert_{\theta = \theta} } \: p_{\pi_\theta}^{\phi^*}(x_{1:T+1}) d(x_{1:T+1})}}_{\mathsf{op}}\\
     \leq \int {\norm { D_\theta \paren{\frac{1}{T}\sum_{t = 1}^{T}{Df(x_t, \pi_\theta(x_t), \phi^*)^\top Df(x_t, \pi_\theta(x_t), \phi^*)} } \bigg\rvert_{\theta = \theta} } _{\mathsf{op}} \: p_{\pi_\theta}^{\phi^*}(x_{1:T+1}) d(x_{1:T+1})}\\
     \leq 2 L_f^2 L_{\Pi}
\end{dmath}
To bound the norm of the second integral, we use Jensen's inequality and the log-derivative trick:\begin{align}
    &\norm{\int {\paren{\frac{1}{T}\sum_{t = 1}^{T}{Df(x_t, \pi_\theta(x_t), \phi^*)^\top Df(x_t, \pi_\theta(x_t), \phi^*)}} \paren{D_\theta \paren{p_{\pi_\theta}^{\phi^*}(x_{1:T+1})I } \bigg\rvert_{\theta = \theta} } \: d(x_{1:T+1})}}_{\mathsf{op}}\\
     &\leq \int{ \norm{\frac{1}{T}\sum_{t=1}^{T}{Df(x_t, \pi_\theta(x_t), \phi^*)^\top Df(x_t, \pi_\theta(x_t), \phi^*)}}_{\mathsf{op}} \norm{\nabla_\theta\paren{p_{\pi_\theta}^{\phi^*}(x_{1:T+1})} \bigg\rvert_{\theta = \theta} } \: d(x_{1:T+1})}\\
     &\leq L_f^2 \int{ \norm{\nabla_\theta\paren{\log p_{\pi_\theta}^{\phi^*}(x_{1:T+1})} \bigg\rvert_{\theta = \theta}} \: p_{\pi_\theta}^{\phi^*}(x_{1:T+1}) d(x_{1:T+1})} \label{expr22}
\end{align}
Then, we substitute (\ref{expr:TrajectoryDensity}) for the density $p_{\pi_\theta}^{\phi^*}(x_{1:T+1})$ and simplify:\begin{align}
   &\eqref{expr22} \notag\\
   &= L_f^2 \int{ \norm{\nabla_\theta\paren{\log\paren{ \paren{\frac{1}{\sigma \sqrt {2\pi}}}^T \exp \paren{-\frac{1}{2\sigma^2}\sum_{t=1}^{T}{\norm{w_t}^2}} }} \bigg\rvert_{\theta = \theta}} \: p_{\pi_\theta}^{\phi^*}(x_{1:T+1}) d(x_{1:T+1})} \\
    &= \frac{L_f^2}{\sigma^2} \int{ \norm{ \sum_{t=1}^{T}{w_t^\top \paren{ D_u f(x_t, \pi_\theta(x_t), \phi^*) } \paren{D_\theta \pi_\theta(x_t) \rvert_{\theta = \theta} } } } \: p_{\pi_\theta}^{\phi^*}(x_{1:T+1}) d(x_{1:T+1})} \label{expr: after density sub}
\end{align}
By Assumptions \ref{assume:SmoothDynamics} and \ref{assume:SmoothPolicyClass}, we can use \Cref{lemma: adapted random walk} (with $U = L_f L_\Pi$) to bound:\begin{align*}
    \eqref{expr: after density sub} \leq \frac{L_f^2}{\sigma^2} \cdot L_f L_\Pi\sigma \sqrt{Td_x} = \frac{L_f^3 L_\Pi \sqrt{T d_x}}{\sigma}
\end{align*}
The desired result follows by combining our bounds for the first and second integrals in \eqref{expr:ThetaDerivativeFI}.
\end{proof}
The following results demonstrate that, for any persistent exciting controller $\pi$, the prediction error of dynamics $\phi$ (using trajectories collected with controller $\pi$) is locally strongly convex near the true dynamics $\phi^*$. These results allow us to analyze the second phase of our algorithm through the lens of online convex optimization.
\begin{lemma}\label{lem:PhiLipschitzFisherInformation}
    Assume that Assumption \ref{assume:SmoothDynamics} holds. Then for any controller $\pi$, and all $\phi \in \mathbb{R}^{d_\phi}$, \begin{align*}
        \norm{D_{\phi}^{(3)} \Err_{\pi}^{\phi^*}(\phi) }_{\mathsf{op}} \leq 6 L_f^2 + 2L_f^2 \norm{\phi - \phi^*} + 2 L_f \sigma \sqrt{d_x}.
    \end{align*}
\end{lemma}
\begin{proof} By the dominated convergence theorem followed by the product rule, the second derivative of the prediction error for dynamics $\phi$ under controller $\pi$, evaluated at $\phi$, is: \begin{dmath}
    D^{(2)}_{\phi} \paren{ \mathbb{E}_{\pi}^{\phi^*} \brac{ \frac{1}{T} \sum_{t=1}^{T} \| f(x_t, u_t, \phi) - x_{t+1} \|^2 } } \bigg|_{\phi = \phi}\\
    = \mathbb{E}_{\pi}^{\phi^*} \brac{ D_{\phi}^{(2)} \paren{ \frac{1}{T}\sum_{t=1}^{T} \| f(x_t, u_t, \phi) - x_{t+1} \|^2 } \bigg|_{\phi = \phi} }\\
    = \mathbb{E}_{\pi}^{\phi^*} \brac{ \frac{1}{T}\sum_{t=1}^{T} \paren{ 2 D_{\phi} f(x_t, u_t, \phi)^\top D_{\phi} f(x_t, u_t, \phi) + 2 (f(x_t, u_t, \phi) - x_{t+1})^\top D_{\phi}^{(2)} f(x_t, u_t, \phi) } }
\end{dmath}
Differentiating this once more with respect to $\phi$, we find:
\begin{dmath}\label{third deriv}
    D^{(3)}_{\Phi} \paren{ \mathbb{E}_{\pi}^{\phi^*} \brac{ \frac{1}{T}\sum_{t=1}^{T} \| f(x_t, u_t, \phi) - x_{t+1} \|^2 }} \bigg|_{\phi = \phi}\\
    = D_{\phi} \paren{ \mathbb{E}_{\pi}^{\phi^*} \brac{ \frac{1}{T}\sum_{t=1}^{T} \paren{ 2 D_{\phi} f(x_t, u_t, \phi)^\top D_{\phi} f(x_t, u_t, \phi) + 2 (f(x_t, u_t, \phi) - x_{t+1})^\top D_{\phi}^{(2)} f(x_t, u_t, \phi) }}} \bigg|_{\phi = \phi} \\
    = \mathbb{E}_{\pi}^{\phi^*} \brac{ \frac{1}{T}\sum_{t=1}^{T} \paren{ 6 D_{\phi} f(x_t, u_t, \phi)^\top D_{\phi}^{(2)} f(x_t, u_t, \phi) + 2 (f(x_t, u_t, \phi) - x_{t+1})^\top D_{\phi}^{(3)} f(x_t, u_t, \phi) }}
\end{dmath}
Using the triangular inequality, Assumption \ref{assume:SmoothDynamics}, and the submultiplicativity of the operator norm:\begin{align}
    \norm{\eqref{third deriv}}_{\mathsf{op}} \leq 6 L_f^2 + 2 L_f \mathbb E_{\pi}^{\phi^*}{ \brac{ \frac{1}{T}\sum_{t=1}^{T}{ \norm{f(x_t, u_t, \phi) - x_{t + 1}} }}}\label{expr27}
\end{align}
Rewriting this, and applying the triangular inequality:
\begin{align}
    \eqref{expr27} &\leq 6L_f^2 + 2 L_f \mathbb{E}_{\pi}^{\phi^*} \brac{ \frac{1}{T}\sum_{t=1}^{T} \norm{ f(x_t, u_t, \phi) - f(x_t, u_t, \phi^*) - w_t } }\\
    &\leq 6L_f^2 + 2 L_f \mathbb{E}_{\pi}^{\phi^*} \brac{ \frac{1}{T}\sum_{t=1}^{T}{ \norm{ f(x_t, u_t, \phi) - f(x_t, u_t, \phi^*) } } + \frac{1}{T}\sum_{t=1}^{T}{ \norm{w_t}} }\label{before expected norm stuff}
\end{align}
Using Assumption \ref{assume:SmoothDynamics} and the expected norm of a Gaussian random variable,\begin{align}
    \eqref{before expected norm stuff} &\leq 6L_f^2 + 2L_f^2 \norm{\phi - \phi^*} + 2 L_f \mathbb{E}_{\pi}^{\phi^*}  \brac{ \frac{1}{T}\sum_{t = 1}^{T}\norm{w_t}}\\
    &\leq 6L_f^2 + 2L_f^2 \norm{\phi - \phi^*} + 2 L_f \sigma \sqrt{d_x}
\end{align}
\end{proof}

\begin{lemma}\label{lem:LocalStrongConvexity}
    Assume that Assumptions \ref{assume:SmoothDynamics}, \ref{assume:SmoothPolicyClass}, \ref{assume:PersistencyOfExcitationOptimalController} hold. Let $\mu$ be defined as in Assumption \ref{assume:PersistencyOfExcitationOptimalController}. Define \begin{align*}
        r_{\mathsf{policy}} \triangleq \frac{\mu}{4L_f^2L_{\Pi} + \paren{\frac{2L_f^3L_{\Pi}\sqrt{T d_x}}{\sigma}} } \quad \text{and}\quad r_{\mathsf{dyn}} \triangleq \min \curly{1, \frac{\mu}{32L_f^2 + 8 L_f\sigma\sqrt{d_x}}}.
    \end{align*}
    Suppose $\theta \in \mathcal B\paren{\theta^*, r_{\mathsf{policy}}}$ and $\phi \in \mathcal B\paren{\phi^*, r_{\mathsf{dyn}}}$. Then,\begin{align*}
        D_{\phi}^{(2)} \Err_{\pi_\theta}^{\phi^*}(\phi) \succeq \frac{\mu}{4}I_{d_\phi}.
    \end{align*}
    In other words, for all $\theta$ sufficiently close to the optimal controller parameters $\theta^*$, the prediction error under $\pi_\theta$ is strongly convex on a neighborhood around $\phi^*$.
\end{lemma}

\begin{proof}
We begin by restating the contents of Assumption \ref{assume:PersistencyOfExcitationOptimalController}:\begin{align}\label{expr:RestateAssumption6}
    D_{\phi}^{(2)} \Err_{\pi_{\theta^*}}^{\phi^*}(\phi^*)\succeq \mu I_{d_\phi}
\end{align}
We proceed by using the smoothness of the error function with respect to the policy parameter, and the operator norm bound proved in Lemma \ref{lem:ThetaLipschitzFisherInformation}, and obtain:\begin{align}\label{expr:AfterLemma3}
    D_{\phi}^{(2)} \Err_{\pi_{\theta}}^{\phi^*}(\phi^*)  \succeq \frac{\mu}{2}I_{d_\phi} \quad\text{for all $\theta \in \mathcal B\left(\theta^*, r_{\mathsf{policy}} \right)$}
\end{align}
Next, note that $r_{\mathsf{dyn}} \leq 1$ implies that $\norm{D_{\phi}^{(3)} \Err_{\pi}^{\phi^*}(\phi)}_{\mathsf{op}} \leq 8L_f^2 + 2L_f \sigma \sqrt{d_x}$ for all $\phi \in \mathcal B(\phi^*, r_{\mathsf{dyn}})$, by Lemma \ref{lem:PhiLipschitzFisherInformation}. We use this operator norm bound and the smoothness of the error function with respect to the dynamics parameter, and conclude:\begin{align}
   D_{\phi}^{(2)} \Err_{\pi_{\theta}}^{\phi^*}(\phi) \succeq \frac{\mu}{4}I_{d_\phi}
    \quad\text{for all $\theta \in \mathcal B\left(\theta^*, r_{\mathsf{policy}} \right)$ and $\phi \in \mathcal B\left(\phi^*, r_{\mathsf{dyn}}\right)$}
\end{align}
\end{proof}

Equipped with these results, we now prove that the estimates $\phi_1, \phi_2, ...$ produced during the online stochastic optimization procedure in the second phase of Algorithm \ref{alg:LogRegret} converge at a fast rate to the true dynamics $\phi^*$. The following descent lemma shows that each gradient step reduces the expected distance between the dynamics estimate $\phi_i$ and the true dynamics $\phi^*$. 

\begin{lemma}\label{lem:GradientOracle}
    Assume that Assumption \ref{assume:SmoothDynamics} holds. Then, for any controller $\pi$ and dynamics estimate $\phi$,\begin{itemize}
        \item $\mathbb{E}\brac{\nabla l_{\mathcal D}(\phi)} = \nabla_\phi \Err_{\pi}^{\phi^*}(\phi)$,

        \item $\mathbb{E}\brac{\norm{\nabla l_{\mathcal D}(\phi)}^2} \leq 4 L_f^4 \norm{\phi - \phi^*}^2 + 8L_f^3 \sigma \sqrt{d_x}\norm{\phi - \phi^*} T^{-1/2} + 4 L_f^2 \sigma^2 d_x T^{-1}$.
    \end{itemize}
    Here, the expectation is taken with respect to a single trajectory $\mathcal D = \{(x_t, \pi(x_t), x_{t+1})\}_{t = 1, ..., T}$ collected using the (deterministic) policy $\pi$. Additionally, we define\begin{align*}
        C_{\mathsf{grad}} \triangleq 16 L_f^2 \sigma^2 d_x T^{-1} \quad\text{and} \quad r_{\mathsf{grad}} = L_f^{-1} \sigma \sqrt{d_x} T^{-1/2}
    \end{align*}
    such that for all $\phi$ with $\norm{\phi - \phi^*} \leq r_{\mathsf{grad}}$, we have $\mathbb{E}\brac{\norm{\nabla l_{\mathcal D}(\phi)}^2} \leq C_{\mathsf{grad}}$.
\end{lemma}

\begin{proof}
    The first claim simply follows from expanding the expectation using (\ref{expr:EmpiricalPredictionError}), then applying the dominated convergence theorem. To prove the second claim, we first expand using \eqref{expr:EmpiricalPredictionError}:\begin{dmath}\label{expr:GradientBound}
        \mathbb{E}\brac{\norm{\nabla l_{\mathcal D}(\phi)}^2}
        = \mathbb{E}\brac{\norm{\nabla_\phi \paren{\frac{1}{T}\sum_{t=1}^{T}{\norm{f(x_t, \pi(x_t), \phi) - x_{t+1}}^2} } \Bigg|_{\phi = \phi} }^2}
        = 4 \mathbb{E} \brac{\norm{\frac{1}{T}\sum_{t=1}^{T}{Df(x_t, \pi(x_t), \phi)^\top(f(x_t, \pi(x_t), \phi) - x_{t+1})} }^2}
    \end{dmath}
    Applying subadditivity, then Assumption \ref{assume:SmoothDynamics}, we upper bound \eqref{expr:GradientBound} as:\begin{align}
        \eqref{expr:GradientBound} &\leq \frac{4}{T^2} \mathbb{E} \brac{ \paren{\norm{\sum_{t=1}^{T}{Df(x_t, \pi(x_t), \phi)^\top(f(x_t, \pi(x_t), \phi) - f(x_t, \pi(x_t), \phi^*)} } + \norm{\sum_{t=1}^{T}{Df(x_t, \pi(x_t), \phi)^\top w_t} } }^2 }\\
        &\leq \frac{4}{T^2} \mathbb{E} \brac{ \paren{T L_f^2 \norm{\phi - \phi^*} + \norm{\sum_{t=1}^{T}{Df(x_t, \pi(x_t), \phi)^\top w_t} } }^2 }\label{after some stuff}
    \end{align}
    Expanding this, then applying \Cref{lemma: adapted random walk}, we upper bound \eqref{after some stuff} as:\begin{align}
        \eqref{after some stuff} &\leq \frac{4}{T^2} \paren{T^2 L_f^4 \norm{\phi - \phi^*}^2 + 2T L_f^2 \sqrt{T d_x} \norm{\phi - \phi^*} + L_f^2 \sigma^2 T d_x}\\
        &= 4 L_f^4 \norm{\phi - \phi^*}^2 + 8L_f^3 \sigma \sqrt{d_x}\norm{\phi - \phi^*} T^{-1/2} + 4 L_f^2 \sigma^2 d_x T^{-1}\label{final}
    \end{align}
    The desired result follows from combining \eqref{expr:GradientBound}-\eqref{final}.
\end{proof}

\begin{lemma}\label{lem:DescentLemma}
Suppose that Assumptions \ref{assume:SmoothDynamics}, \ref{assume:SmoothPolicyClass}, and \ref{assume:PersistencyOfExcitationOptimalController} hold. Let $\mu$ be defined as in Assumption \ref{assume:PersistencyOfExcitationOptimalController}, let $r_{\mathsf{policy}}$ and $r_{\mathsf{dyn}}$ be defined as in \Cref{lem:LocalStrongConvexity}, and let $r_{\mathsf{grad}}$ be defined as in \Cref{lem:GradientOracle}. Let $\theta \in B(\theta^*, r_{\mathsf{policy}})$. Let $\Phi \subseteq\calB(\phi^*, \min \curly{r_{\mathsf{dyn}}, r_{\mathsf{grad}}})$ be a convex set such that $\phi^* \in \Phi$, and fix some $\phi_i \in \Phi$. Finally, let $\eta_i > 0$ be some step size. 

Next, define the random variables $\psi_{i +1}$ and $\phi_{i +1}$ such that \begin{itemize}
    \item $\psi_{i + 1} = \phi_{i} - \eta_i \nabla l_{\mathcal D}(\phi_{i})$,

    \item $\phi_{i + 1} = \argmin_{\phi \in \Phi}{\norm{ \phi - \psi_{i + 1} }}$.
\end{itemize}
Then, \begin{align*}
    \mathbb{E}\brac{\norm{\phi_{i + 1} - \phi^* }^2} \leq C_{\mathsf{grad}} \eta_i^2 + \left(1 - \frac{\mu}{4} \eta_i\right) \mathbb{E}\brac{ \norm{\phi_{i} - \phi^*}^2 }.
\end{align*}
where $C_{\mathsf{grad}}$ is defined in Lemma \ref{lem:GradientOracle}. Here, the expectation is taken with respect to a single trajectory $\mathcal D = \curly{(x_t, \pi_\theta(x_t), x_{t + 1})}_{t = 1, \dots, T}$ collected using the (deterministic) policy $\pi_\theta$.

% Here, the randomness is taken with respect to $P_\theta$, the distribution of trajectories under $\pi_\theta$.
\end{lemma}

\begin{proof}
We begin by recalling Lemma \ref{lem:LocalStrongConvexity}, which states that $\Err_{\pi_\theta}^{\phi^*}(\phi)$ is $\frac{\mu}{4}$-strongly convex on $\mathcal B(\phi^*, r_{\mathsf{dyn}})$ as long as $\theta \in \mathcal B(\theta^*, r_{\mathsf{policy}})$. By the first-order condition for a $\frac{\mu}{4}$-strongly convex function,\begin{dmath}\label{expr:FirstOrderStrongConvexityCondition}
    \Err_{\pi_{\theta}}^{\phi^*}(\phi_2) \geq \Err_{\pi_{\theta}}^{\phi^*}(\phi_1 ) + (\nabla_\phi \Err_{\pi_{\theta}}^{\phi^*}(\phi_1))^\top (\phi_2 - \phi_1) + \frac{\mu}{8}\|\phi_2 - \phi_1 \|^2\quad\text{for all $\phi_1, \phi_2 \in\calB(\phi^*, r_{\mathsf{dyn}})$}
\end{dmath}
Rearranging (\ref{expr:FirstOrderStrongConvexityCondition}) with $\phi_1 = \phi_i$ and $\phi_2 = \phi^*$, and applying Lemma \ref{lem:GradientOracle}, we have that for any fixed values of $\phi_{i}, \psi_{i + 1}$, and $\phi_{i + 1}$: \begin{dmath}\label{expr:FirstOrderConvexityExprRearranged}
    \Err_{\pi_\theta}^{\phi^*}(\phi_{i}) \leq \Err_{\pi_\theta}^{\phi^*}(\phi^*) - \paren{\nabla_\phi \Err_{\pi_\theta}^{\phi^*}(\phi_{i})}^\top (\phi^* - \phi_{i}) - \frac{\mu}{8}\norm{\phi^* - \phi_{i} }^2\\
    = \Err_{\pi_\theta}^{\phi^*}(\phi^*) - \paren{\mathbb{E}\brac{\nabla l_{D_{\theta}}(\phi_{i})}}^\top (\phi^* - \phi_{i}) - \frac{\mu}{8}\norm{\phi^* - \phi_{i} }^2\\
    = \Err_{\pi_\theta}^{\phi^*}(\phi^*) - \frac{1}{\eta_i} \paren{\mathbb{E}\brac{ \phi_{i} - \psi_{i + 1} }}^\top (\phi^* - \phi_{i}) - \frac{\mu}{8}\norm{\phi^* - \phi_{i} }^2.
\end{dmath}
Applying the three-point identity $2(c-b)^\top(a-b) = \norm{b - a}^2 - \norm{c-a}^2 + \norm{c-b}^2$ with $a = \phi^*$, $b = \phi_i$, and $c = \psi_{i + 1}$, we rearrange \eqref{expr:FirstOrderConvexityExprRearranged} to obtain:\begin{dmath}\label{expr:AfterThreePointInequality}
    \Err_{\pi_\theta}^{\phi^*}(\phi_{i}) \leq \Err_{\pi_\theta}^{\phi^*}(\phi^*) + \frac{1}{2 \eta_i}\cdot \mathbb{E}\brac{ \norm{\phi_{i} - \phi^*}^2 - \norm{\psi_{i + 1} - \phi^* }^2 + \norm{\phi_{i} - \psi_{i + 1} }^2 } - \frac{\mu}{8}\norm{\phi^* - \phi_{i} }^2.
\end{dmath}
Next, observe that $\phi^* \in \argmin_{\phi \in \mathbb{R}^{d_\phi}}{\Err_{\pi}^{\phi^*}(\phi)}$ for all controllers $\pi$ (this fact is a result of the bias-variance decomposition of the square loss). Using this to rearrange \eqref{expr:AfterThreePointInequality}, we obtain:\begin{align}
    \mathbb{E}\brac{\norm{\psi_{i + 1} - \phi^*}^2}
    &\leq 2 \eta_i \paren{\Err_{\pi_{\theta}}^{\phi^*}(\phi^*) - \Err_{\pi_\theta}^{\phi^*}(\phi_{i})} + \mathbb{E}\brac{\norm{\phi_{i} - \phi^*}^2 + \norm{ \phi_{i} - \psi_{i + 1} }^2 } - \frac{\mu}{4}\eta_i\norm{\phi^* - \phi_{i} }^2 \notag \\
    &\leq \mathbb{E}\brac{ \paren{ 1- \frac{\mu}{4}\eta_i}\norm{\phi_{i} - \phi^*}^2 + \norm{\phi_{i} - \psi_{i + 1} }^2 }.\label{expr:ClosenessBound1}
\end{align}
Next, we write $\mathbb{E}\brac{\norm{ \phi_{i} - \psi_{i + 1} }^2}$ as the second moment of the gradient estimator, and bound with Lemma \ref{lem:GradientOracle} (also using the fact that the radius is $\min \curly{r_{\mathsf{dyn}}, r_{\mathsf{grad}}} \leq r_{\mathsf{grad}}$):\begin{align}\label{expr:ClosenessBound2}
    \mathbb{E}\brac{\norm{ \phi_{i} - \psi_{i + 1} }^2} = \eta_i^2 \mathbb{E}\brac{\norm{ \nabla l_{D_{\theta}}(\phi_{i}) }^2} \leq C_{\mathsf{grad}} \eta_i^2.
\end{align}
Combining \eqref{expr:ClosenessBound1} and \eqref{expr:ClosenessBound2}, and using that $\phi_{i + 1}$ is the projection of $\psi_{i + 1}$ onto a convex set $\Phi$ containing $\phi^*$ (which implies that $\norm{\phi_{i + 1} - \phi^*} \leq \norm{\psi_{i + 1} - \phi^*}$), we conclude:\begin{align}
    \mathbb{E}\brac{\norm{ \phi_{i + 1} - \phi^* }^2} \leq \mathbb{E}\brac{\norm{ \psi_{i + 1} - \phi^* }^2} \leq C_{\mathsf{grad}} \eta_i^2 + \left(1 - \frac{\mu}{4} \eta_i\right) \mathbb{E}\brac{ \norm{\phi_{i} - \phi^*}^2 }
\end{align}
\end{proof}

\subsection{Proof of Theorem \ref{thm:LogRegret}}

We will now proceed to prove our main result, the polylogarithmic regret bound for Algorithm \ref{alg:LogRegret}.

Suppose the confidence radius $r_{\Phi}$ given to Algorithm \ref{alg:LogRegret} satisfies:\begin{align}\label{rPhi condition}
    r_{\Phi} \leq \frac 1 2 \min \curly{ r_{\mathsf{ce}}, r_{\mathsf{cost}}, \frac{r_{\mathsf{policy}}}{L_{\mathsf{ce}}}, r_{\mathsf{dyn}}, r_{\mathsf{grad}} }
\end{align}
and that the number of initial phase episodes $N_{\mathsf{phase\,1}}$ given to Algorithm \ref{alg:LogRegret} satisfies:\begin{align}\label{nPhase1 condition}
    N_{\mathsf{phase\,1}} \geq \max \curly{ \tau_{\mathsf{err}} (\delta), \frac{C_\Err \sigma^2 \left( d_x + d_\phi \log \left(\frac{L_f TN}{\delta}\right) \right)}{T{r_{\Phi}}^{1/\alpha}} } \quad \text{for some $\delta  > 0$}
\end{align}
Inverting Lemma \ref{lem: er bound}, we can show that with the above choices of $N_{\mathsf{phase\,1}}$ and $r_\Phi$, we have that with probability at least $1 - \delta$, the following hold simultaneously: \begin{itemize}
    \item $\phi^* \in \Phi$;

    \item $\Phi \subseteq\calB(\phi^*, \min \{ r_{\mathsf{ce}}, r_{\mathsf{dyn}}, \frac{r_{\mathsf{policy}}}{L_{\mathsf{ce}}} \})$, which together with Assumption \ref{assume:WagenmakerProp6}, implies that the local strong convexity guarantees of Lemma \ref{lem:LocalStrongConvexity} hold;

    \item $\Phi \subseteq\calB(\phi^*, r_{\mathsf{cost}})$, which implies that $\Phi$ only contains policies close enough to $\phi^*$ in order to bound the suboptimality $J(\pi_{\theta^*(\phi)}, \phi^*) - J(\pi_{\theta^*}, \phi^*)$ as a quadratic function of $\|\phi - \phi^*\|$.

    \item $\Phi \subseteq \mathcal B(\phi^*, r_{\mathsf{grad}})$, which by \Cref{lem:GradientOracle}, implies that the second moment of the gradient estimator is bounded by $C_{\mathsf{grad}}$, so that we may apply \Cref{lem:DescentLemma}.
\end{itemize}
Denote by $\calE_{\mathsf{phase\,1}}$ the event that the above are satisfied after the first phase of Algorithm \ref{alg:LogRegret}. We will analyze the regret of Algorithm \ref{alg:LogRegret} under the events $\calE_{\mathsf{phase\,1}}^\complement$ and $\calE_{\mathsf{phase\,1}}$. 

\paragraph{If the first phase fails.} Conditioning on $\calE_{\mathsf{phase\,1}}^\complement$, we can apply Assumption \ref{assume:BoundedCosts} for a crude bound: \begin{align}\label{expr:FailureBound}
    \mathbb E [\Regret(N) \mid \calE_{\mathsf{phase\,1}}^\complement] \leq N T L_{\mathsf{cost}}.
\end{align}

\paragraph{If the first phase succeeds.} Conditioning on $\calE_{\mathsf{phase\,1}}$, we unroll Lemma \ref{lem:DescentLemma} to show that the estimates $\phi_0, \phi_1, \dots, \phi_i, \dots$ converge at a fast $\frac{1}{i}$ rate to the true dynamics $\phi^*$ in the second phase.

We will solve for $\eta_i$ and a constant $A \geq 0$ such that for all $i = 1, 2, \dots$, it holds that if $\mathbb{E}\brac{\norm{\phi_i - \phi^\star}^2 \mid \calE_{\mathsf{phase\,1}}} \leq \frac{A}{i}$ and we choose a step size of $\eta_i$ for the $i^{th}$ gradient step, then $\mathbb{E}\brac{\norm{\phi_{i + 1} - \phi^\star}^2 \mid \calE_{\mathsf{phase\,1}}} \leq \frac{A}{i + 1}$. Here, the expectations are taken over the randomness of the gradient oracle and the system noise, and conditioned on the success of the first phase. The proof follows the standard analysis of stochastic gradient descent, with the standard descent lemma replaced by our Lemma \ref{lem:DescentLemma}. Applying Lemma \ref{lem:DescentLemma}:\begin{dmath}\label{expr:AUpperBound}
    \mathbb{E}\brac{\norm{\phi_{i + 1} - \phi^\star}^2 \mid \calE_{\mathsf{phase\,1}}}
    \leq C_{\mathsf{grad}}\eta_i^2 + \paren{1 - \frac{\mu}{4}\eta_i} {\mathbb{E}\brac{\norm{\phi_i - \phi^\star}^2 \mid \calE_{\mathsf{phase\,1}}}}
    \leq C_{\mathsf{grad}}\eta_i^2 + \paren{1 - \frac{\mu}{4}\eta_i}\frac{A}{i}.
\end{dmath}

To show that $\mathbb{E}\brac{\norm{\phi_{i + 1} - \phi^\star}^2 \mid \calE_{\mathsf{phase\,1}}} \leq \frac{A}{i + 1}$, it suffices to show $C_{\mathsf{grad}}\eta_i^2 + \paren{1 - \frac{\mu}{4}\eta_i}\frac{A}{i} \leq \frac{A}{i + 1}$. Rearranging, it can be show that the choices $\eta_i = \frac{8}{\mu i}$ and $A \geq \frac{64 C_{\mathsf{grad}}}{\mu^2}$ suffices. To explicitly bound $\mathbb{E}\brac{\norm{\phi_{i} - \phi^\star}^2 \mid \calE_{\mathsf{phase\,1}}}$, we note that the success of the first phase implies that $\norm{\phi_1 - \phi^*} \leq 2r_{\Phi}$. Combining results yields the explicit upper bound:\begin{align}\label{expr:DynamicsErrorBound}
    \mathbb{E}\brac{\norm{\phi_i - \phi^\star}^2 \mid \calE_{\mathsf{phase\,1}}} \leq \frac{1}{i} \max \left\{ \frac{64 C_{\mathsf{grad}}}{\mu^2}, 4r_{\Phi}^2 \right\}.
\end{align}
i.e., conditioning on the event $\calE_{\mathsf{phase\,1}}$, the dynamics estimation error $\mathbb{E}\brac{\norm{\phi_{i} - \phi^\star}^2 \mid \calE_{\mathsf{phase\,1}}}$ decays as $O(\frac{1}{i})$ in the second phase of Algorithm \ref{alg:LogRegret}.

Next, we show that an $O(\frac{1}{i})$ rate of dynamics estimation error translates to an $O(\log i)$ cumulative regret rate during the second phase of Algorithm \ref{alg:LogRegret}. By the success of the first phase (in particular, that $\Phi =\calB(\phi_0, r_{\Phi}) \subseteq\calB(\phi^*, \min \{r_{\mathsf{cost}}, r_{\mathsf{ce}}\})$, we may apply Lemma \ref{lem: cost decomposition} to obtain:\begin{dmath}
    J(\pi_{\theta^*(\phi_i)}, \phi^*) - J(\pi_{\theta^*}, \phi^*) \leq C_{\mathsf{cost}} \norm{\hat\phi - \phi^*}^2
\end{dmath}
Taking conditional expectations (with respect to $\calE_{\mathsf{phase\,1}}$) on both sides, and combining with (\ref{expr:DynamicsErrorBound}):\begin{align}
    \mathbb{E}\brac{J(\pi_{\theta^*(\phi_i)}, \phi^*) - J(\pi_{\theta^*}, \phi^*) \mid \calE_{\mathsf{phase\,1}}}
    \leq \frac{C_{\mathsf{cost}}}{i} \max \curly{\frac{64 C_{\mathsf{grad}}}{\mu^2}, 4r_{\Phi}^2 }
\end{align}
We thus conclude: \begin{dmath}\label{expr:SuccessBound}
    \mathbb{E}\brac{\Regret(N) \mid \calE_{\mathsf{phase\,1}}} \leq { N_{\mathsf{phase\,1}}TL_{\mathsf{cost}} + \sum_{i = 1}^{N - N_{\mathsf{phase\,1}}}{\paren{\frac{C_{\mathsf{cost}}}{i} \max \left\{ \frac{64 C_{\mathsf{grad}}}{\mu^2}, 4r_{\Phi}^2 \right\} } } }
    \leq { N_{\mathsf{phase\,1}}TL_{\mathsf{cost}} + C_{\mathsf{cost}} \max \left\{ \frac{64 C_{\mathsf{grad}}}{\mu^2}, 4r_{\Phi}^2 \right\} \log N }.
\end{dmath}

\paragraph{Choosing $\delta$ to balance the regret.} Finally, we find a specific choice of $\delta$ such that $\mathbb{E}[\Regret(N)]$ is logarithmic in $N$. We begin by conditioning with respect to $\calE_{\mathsf{phase\,1}}$ and applying (\ref{expr:FailureBound}) and (\ref{expr:SuccessBound}):\begin{align}
    \mathbb{E}[\Regret(N)] 
    &= (1 - \mathbb{P}(\calE_{\mathsf{phase\,1}}))\mathbb{E}[{\Regret(N) \mid \calE_{\mathsf{phase\,1}}^\complement}] + \mathbb{P}(\calE_{\mathsf{phase\,1}})\mathbb{E}[{\Regret(N) \mid \calE_{\mathsf{phase\,1}}}] \notag\\
    &\leq \delta N T L_{\mathsf{cost}} + N_{\mathsf{phase\,1}}T L_{\mathsf{cost}} + C_{\mathsf{cost}} \max \curly{ \frac{64 C_{\mathsf{grad}}}{\mu^2}, 4r_{\Phi}^2 } \log N.\label{expr:BeforeChoosingDelta}
\end{align}
Here, we take a moment to remark that the coefficient $C_{\mathsf{cost}} \max \curly{ \frac{64 C_{\mathsf{grad}}}{\mu^2}, 4r_{\Phi}^2 }$ has a linear dependence on $T$ (and polynomial in all other relevant system constants), since $C_{\mathsf{cost}} = O(T^2)$; $C_{\mathsf{grad}} = O(T^{-1})$ by \Cref{lem:GradientOracle}; and $4r_{\Phi}^2 = O(T^{-1})$ since $r_\Phi \leq \frac{r_{\mathsf{policy}}}{L_{\mathsf{ce}}} = O(T^{-1/2})$.

To finish off the proof, we conclude with the natural choice $\delta = 1/N$. Then, (\ref{expr:BeforeChoosingDelta}) becomes:\begin{align*}\label{expr:AfterChoosingDelta}
    \mathbb{E}[\Regret(N)] &\leq TL_{\mathsf{cost}} + N_{\mathsf{phase\,1}}TL_{\mathsf{cost}} + C_{\mathsf{cost}} \max \curly{ \frac{64 C_{\mathsf{grad}}}{\mu^2}, 4r_{\Phi}^2 }\log N \\
    &= \mathsf{poly}_\alpha (d_x, \sigma, L_f, L_{\mathsf{cost}}, \mu^{-1}) T \log N + TN_{\mathsf{phase\,1}}L_{\mathsf{cost}}
\end{align*}
as long as \eqref{rPhi condition} and \eqref{nPhase1 condition} hold with $\delta = 1/N$, i.e., \begin{align*}
    r_{\Phi} &\leq \frac 1 2 \min \curly{ r_{\mathsf{ce}}, r_{\mathsf{cost}}, \frac{r_{\mathsf{policy}}}{L_{\mathsf{ce}}}, r_{\mathsf{dyn}}, r_{\mathsf{grad}} } = \mathsf{poly}\paren{r_{\mathsf{ce}}, r_{\mathsf{cost}}, d_x^{-1}, \sigma^{-1}, L_f^{-1}, L_{\Pi}^{-1}, L_{\mathsf{ce}}^{-1}, \mu} T^{-1/2}
\end{align*}
and\begin{align*}
    N_{\mathsf{phase\,1}} &\geq \max \curly{ \tau_{\mathsf{err}} \paren{1/N}, \frac{C_\Err \sigma^2 \paren{ d_x + 2 d_\phi \log \left(L_f TN\right) }}{T{r_{\Phi}}^{1/\alpha}} } \\
   &= \mathsf{poly}_\alpha \paren{\log N, \log T, d_x, d_\phi, \sigma, C_{\mathsf{Loja}}, r_{\mathsf{ce}}^{-1}, r_{\mathsf{cost}}^{-1}, L_f, L_\Pi, L_{\mathsf{ce}}, \mu^{-1}, \log B } T^{1 / (2\alpha) - 1}
\end{align*}

% \section{Improving the Dependence on \texorpdfstring{$T$}{T}}\label{appendix T dependence}

\end{document}